\RequirePackage{fix-cm}
\documentclass[]{article}
\usepackage{arxiv}
\pdfoutput=1
\usepackage[english]{babel}
\usepackage{graphicx}
\usepackage{epstopdf}
\usepackage{placeins}
\usepackage{bm}
\usepackage{booktabs} 
\usepackage{csquotes}
\usepackage{subfigure}
\usepackage{multirow}
\usepackage{listings}
\usepackage{xcolor}
\usepackage{float}
\usepackage{color, colortbl}
\definecolor{ao(english)}{rgb}{0.0, 0.5, 0.0}
\usepackage{xr-hyper}
\usepackage{hyperref}
\hypersetup{colorlinks,linkcolor={blue},citecolor={ao(english)},urlcolor={red}} 
\usepackage{stackrel,amssymb}
\usepackage{amsfonts,amsmath,amstext,amsbsy,amsthm}
\theoremstyle{definition}
\newtheorem{prop}{Proposition}
\newtheorem{definition}{Definition}[section]
\newtheorem{theorem}{Theorem}[section]

\usepackage{cleveref}
\usepackage{amsmath}
\usepackage{verbatim}
\definecolor{LightCyan}{rgb}{0.88,1,1}
\definecolor{LightRed}{rgb}{1,0.7,0.7}
\definecolor{babyblueeyes}{rgb}{0.63, 0.79, 0.95}
\definecolor{azure(colorwheel)}{rgb}{0.0, 0.5, 1.0}
\def\xb{{\bm {x}}}
\newcommand{\R}{\mathbb{R}}

\newcommand{\N}{\mathbb{N}}

\newcommand{\correspondingauthor}{Corresponding author at: Dipartimento di Matematica e Applicazioni “Renato Caccioppoli”, Università degli Studi di Napoli\\ \indent \indent “Federico II”, Naples 80126, Italy. \, 
{\texttt{constantinos.siettos@unina.it }}}
\newcommand{\allaffiliations}{
\indent \textcolor{teal}{$^{(1)}$} Modelling Engineering Risk and Complexity, \emph{Scuola Superiore Meridionale}, Naples 80138, Italy \hspace{1cm}\\
\indent \textcolor{teal}{$^{(2)}$} Department of Chemical and Biomolecular Engineering, \emph{Johns Hopkins University}, Baltimore 21210, MD, USA\\
\indent \textcolor{teal}{$^{(3)}$} Dipartimento di Matematica e Applicazioni ‘‘Renato Caccioppoli", \emph{Universit\`a degli Studi di Napoli} \emph{Federico II}, Naples 80126, Italy\\
\indent \textcolor{teal}{$^{(4)}$}Department of Statistics, \emph{Athens University of Economics and Business}, Athens, Greece\\
\indent \textcolor{teal}{$^{(5)}$} Department of Applied Mathematics and Statistics, \emph{Johns Hopkins University}, Baltimore 21210, MD, USA\\
\indent \textcolor{teal}{$^{(6)}$} Medical School, Department of Urology, \emph{Johns Hopkins University}, Baltimore 21210, MD, USA
}

\title{RandONet: Shallow-Networks with Random Projections for learning linear and nonlinear operators}

\makeatletter
\let\newtitle\@title
\makeatother

\author{
\textbf{Gianluca Fabiani}\textcolor{teal}{$^{1}$}, \,
\textbf{Ioannis G. Kevrekidis}\textcolor{teal}{$^{2,5,6}$}, \,
\textbf{Constantinos Siettos}\textcolor{teal}{$^{3,}$}\thanks{\correspondingauthor \\
\allaffiliations
} \, ,
\textbf{Athanasios N.  Yannacopoulos}\textcolor{teal}{$^{4}$}
}

\usepackage{amsopn}

\begin{document}

\normalsize
\maketitle
    



\begin{abstract}
%
Deep Operator Networks (DeepOnets) have revolutionized the domain of scientific machine learning for the solution of the inverse problem for dynamical systems. 
However, their implementation necessitates optimizing a high-dimensional space of parameters and hyperparameters. This fact, along with the requirement of substantial computational resources, poses a barrier to achieving high numerical accuracy.
Here, inpsired by DeepONets and to address the above challenges, we present Random Projection-based Operator Networks (RandONets): shallow networks with random projections that learn linear and nonlinear operators. The implementation of RandONets involves: (a) incorporating random bases, thus enabling the use of shallow neural networks with a single hidden layer, where the only unknowns are the output weights of the network's weighted inner product; this reduces dramatically the dimensionality of the parameter space; and, based on this, (b) using established least-squares solvers (e.g., Tikhonov regularization and preconditioned QR decomposition) that offer superior numerical approximation properties compared to other optimization techniques used in deep-learning.
In this work, we prove the universal approximation accuracy of RandONets for approximating nonlinear operators and demonstrate their efficiency in approximating linear nonlinear evolution operators (right-hand-sides (RHS)) with a focus on PDEs. 
We show, that for this particular task, RandONets outperform, both in terms of numerical approximation accuracy and computational cost, the ``vanilla" DeepOnets.
\end{abstract}


%
%

\section{Introduction}
In recent years, significant advancements in machine learning (ML) have broadened our computational toolkit with the ability to solve both the forward and, importantly, the inverse problem in differential equations and multiscale/complex systems. For the forward problem, ML algorithms such as Gaussian process and physics-informed neural networks (PINNs) are trained to approximate the solutions of nonlinear differential equations, with a particular interest in stiff and high-dimensional systems of nonlinear differential equations \cite{han2018solving,lu2021deepxde,raissi2019physics,fabiani2021numerical,calabro2021extreme,fabiani2023parsimonious,dong2021computing,dong2021local}, as well as for the solution of nonlinear functional equations 
\cite{karniadakis2021physics,lu2021learning,kalia2021learning,alvarez2023discrete,patsatzis2023slow, alvarez2024nonlinear}. The solution of the inverse problem leverages the ability of ML algorithms to learn the physical laws, their parameters and closures among scales from data \cite{raissi2018deep,raissi2019physics,lee2020coarse,karniadakis2021physics,galaris2022numerical,fabiani2024task,lee2023learning,dietrich2023learning,lu2021learning,kalia2021learning}. To the best of our knowledge, the first neural network-based solution of the inverse problem for identifying the evolution law (the right-hand-side) of parabolic Partial Differential Equations (PDEs), using spatial partial derivatives as basis functions, was presented in Gonzalez et al. (1998) \cite{gonzalez1998identification}. In the same decade, such inverse identification problems for PDEs, were investigated through reduced order models (ROMs) for PDEs, using data-driven Proper Orthogonal Decomposition (POD) basis functions \cite{krischer1993heterogeneously} and Fourier basis functions (in a context of approximate inertial manifolds) in \cite{shvartsman2000order}.\par


Over the last few years several advanced ML-based approaches for the approximation of nonlinear operators, focused on partial differential operators, stand out: the Deep Operator Networks (DeepOnet) \cite{lu2021learning}, the Fourier Neural Operators (FNOs) \cite{li2020Fourier}, and the Graph-based Neural Operators \cite{kovachki2023neural,li2020multipole} are the most prominent ones.
The DeepOnet extends the universal approximation theorem for dynamical systems --given back in 90's by Chen and Chen \cite{chen1995universal}-- employing the so-called branch and trunk networks that handle input functions and spatial variables and/or parameters separately, thus providing a powerful and versatile framework for operator learning in dynamical systems. Various architectures can be used for either/both networks, thus offering new avenues for tackling challenging problems in the modeling of complex dynamical systems \cite{lu2021learning,lu2022comprehensive,goswami2022physics,jin2022mionet}.
FNOs \cite{li2020Fourier} exploit the Fourier transform to capture global patterns and dependencies in the data. FNOs employ convolutional layers in the frequency domain and the inverse Fourier transform to map back to the original domain. This transformation allows the neural network to efficiently learn complex, high-dimensional inputs and outputs with long range correlations. This method is particularly advantageous for problems involving large spatial domains. 
The family of graph-based neural Operators \cite{li2020multipole,kovachki2023neural} model the nonlinear operator as a graph  --where nodes represent spatial locations of the output function-- learning the kernel of the network which approximates the PDE. They define a sequence of compositions, where each layer is a map between infinite-dimensional spaces with a finite-dimensional parametric dependence. The authors also prove a universal approximation theorem showing that the formulation can approximate any given nonlinear continuous operator. Building on the above pioneering methods, other approaches include wavelet neural operators (WNOs) \cite{tripura2023wavelet} and spectral neural operators (SNOs) \cite{fanaskov2023spectral} using a fixed number of basis functions for both the input and the output functions, which can be either Chebyshev polynomials or complex exponentials. For a comprehensive review on the applications of neural operators, the interested reader can refer to the recent work in \cite{azizzadenesheli2024neural}.\par
Here, based on the architecture of DeepOnets \cite{lu2022learning} and the universal approximation theorem proposed in \cite{chen1995universal}, we present {\em Random Projection-based Operator Network} (RandONet) to deal with the ``curse of dimensionality'' in the training process of DeepONet. DeepONet, while a powerful methodology, is not without its limitations. Its training often involves iterating over large datasets multiple times to update the high-dimensional space of the deep learning network parameters, requiring significant computational time and memory. Additionally, the complexity of the underlying nonlinear operators and the size of the problem domain can further increase the computational burden. Moreover, hyperparameter tuning, regularization techniques, and model selection procedures contribute to additional computational overhead. As a result, training DeepONet can require substantial computational resources, including high-performance computing clusters or GPUs. Importantly, the computational demands of training DeepONet can significantly impact convergence behavior and numerical approximation accuracy. The high-dimensional parameter space may lead to challenges in converging to a (near) global optimum. In some cases, the optimization algorithm may get stuck in local minima or plateaus, hindering the network's ability to approximate the underlying nonlinear or even linear, as we will show, operators with a high accuracy. Addressing these challenges requires careful consideration of optimization strategies, regularization methods and dataset size, balancing computational efficiency with the desired level of approximation accuracy (see also critical discussions and approaches to deal with this cost-accuracy tradeoff in \cite{wang2021learning,de2022cost,venturi2023svd,goswami2023physics}). \par
Our proposed RandONet deals with these challenges, leveraging shallow-neural networks and random projections \cite{johnson1984extensions,rahimi2007random,rahimi2008uniform,rahimi2008weighted,gorban2016approximation} to enable a computationally efficient framework for the approximation of linear and nonlinear operators.
Additionally, we integrate established niche numerical analysis techniques for the solution of the resulting linear least-squares problem, such as Tikhonov regularization, and pivoted QR decomposition with regularization, offering highly efficient, non-iterative solutions with guaranteed (near) optimal estimations. Here, we present RandONets as universal approximators of linear and nonlinear operators. Furthermore, we assess their performance with a focus on the approximation of linear and nonlinear evolution operators (right-hand-sides (RHS)) PDEs. For our illustrations, we consider the 1D viscous Burgers PDE, and the 1D phase-field Allen-Cahn PDE. We demonstrate that, for the particular task of the approximation of evolution operators (RHS) of PDEs, RandOnets outperform the vanilla DeepOnets both in terms of numerical approximation accuracy and computational cost by orders of magnitude. In a work that will follow, we will present the efficiency of RandONets to approximate solution operators of PDEs. Overall, our work contributes to a deeper understanding of DeepOnets, improving significantly their potential for approximating faster, and more accurately, nonlinear operators by incorporating linear and nonlinear random embeddings and established niche methodologies of numerical analysis for the solution of the inverse problems, with a particular focus in PDEs and complex systems.\par 
The paper is organized as follows. In Section \ref{sec:problem}, we describe the problem. In Sections \ref{sec:pre_deepOnets} and \ref{sec:RPNN}, we present the preliminaries regarding the fundamentals and necessary notation for DeepOnets and the random projection neural networks (RPNNs) approaches, respectively. In section \ref{sec:RandONets} we introduce RandONets, and then, in Section \ref{sec:theory}, we extend the theorem of Chen and Chen \cite{chen1995universal} on the universal approximation of Operator to RandONets architectures.
In Section \ref{sec:numerical}, we assess the performance of RandONets and various linear and nonlinear benchmark problems and compare its performance with the vanilla DeepONet. We start with some simple problems of ODEs, where we approximate the solution operator, and then we proceed with the presentation of the results on the approximation of the evolution operator of PDEs.
We conclude, thus giving future perspectives, in Section \ref{sec:conclusion}.

\section{Description of the problem}
\label{sec:problem}
In this study, we focus on the challenging task of learning linear and nonlinear functional operators $\mathcal{F}:\mathsf{U} \rightarrow \mathsf{V}$ which constitute maps between two infinite-dimensional function spaces $\mathsf{U}$ and $\mathsf{V}$. Here, for simplicity, we consider both $\mathsf{U}$ and $\mathsf{V}$ to be subsets of the set $\mathsf{C}(\R^d)$ of continuous functions on $\R^d$. The elements of the set $\mathsf{U}$ are functions $u:\mathsf{X}\subseteq \R^d\rightarrow \R$ that are transformed to other functions $v=\mathcal{F}[u]:\mathsf{Y}\subseteq \R^d \in \R$ through the application of the operator $\mathcal{F}$. We use the following notation for an operator evaluated at a location $\bm{y} \in \mathsf{Y}\subseteq \mathbb{R}^d$
\begin{equation}
    v(\bm{y})=\mathcal{F}[u](\bm{y}).
\end{equation}
These operators play a pivotal role in various scientific and engineering applications, particularly in the context of (partial) differential equations.
By effectively learning (discovering from data) such nonlinear operators, we seek to enhance our understanding and predictive capabilities in diverse fields, ranging from fluid dynamics and materials science to financial and biological systems and beyond \cite{kevrekidis2003equation,lu2021learning,karniadakis2021physics,wang2021learning,goswami2022physics,galaris2022numerical,papaioannou2022time,lee2023learning,gallos2024data,fabiani2024task}.
One prominent example is the right-hand side (RHS) evolution operators $\mathcal{L}$ associated with differential equations (PDEs), which govern the temporal evolution of the associated system dynamics. We denote these \textit{evolution} operators in the following way:
\begin{equation}
v(\bm{x},t)=\frac{\partial u(\bm{x},t)}{\partial t}=\mathcal{L}[u](\bm{x},t), \qquad \bm{x}\in\Omega, \quad t\in[0,T],
\end{equation}
where $u:\mathsf{\Omega} \times [0,T] \subseteq \R^d \rightarrow \R$ is the unknown solution of the PDE (methods for the identification of such PDEs and in general RHS of dynamical systems with ML can be traced back to the '90s \cite{krischer1993heterogeneously,patra1999identification,siettos1999advanced,shvartsman2000order,siettos2002truncated}).
Given a state profile $u(\cdot,t):\Omega\rightarrow \R$ at each time $t$, e.g., the initial condition $u_0$ of the system at time $t=0$, the {\em evolution operator}  (the right-hand-side) of differential equations) $\mathcal{L}$ provides the corresponding time derivative (the output $v(\cdot,t)$) of the system at that time $t$. Again, a method for learning the RHS of PDEs with a different ANN architecture than DeepONet was proposed back in '90s in \cite{gonzalez1998identification}. There, the RHS was estimated in terms of spatial derivatives.\par
One can also learn the corresponding {\em solution operators} $\mathcal{S}_t$, which embody both the time integration and the satisfaction of boundary conditions, of the underlying physical phenomena. Given the initial condition $u_0$ at time $t=0$, the solution operator outputs the state profile $u(\cdot,t):\mathsf{\Omega}\rightarrow \R$ after a certain amount of time $t$:
\begin{equation}
v(\bm{x})=u(\bm{x},t)=\mathcal{S}_t[u_0](\bm{x}).
\end{equation}
We will deal with this problem in the part  II that will follow.\par

Although our objective is to learn functional operators from data, which take functions ($u$) as input, 
we must discretize them to effectively represent them 
and be able to apply network approximations. One practical approach, as implemented in the DeepOnet framework, is to use the function values ($u(\bm{x}_j)$) at a sufficient, but finite, number of locations ${\bm{x}_1, \bm{x}_2, \dots , \bm{x}_m}$, where $\bm{x}_j \in \mathsf{X}\subseteq\R^d$; these locations are referred to as ``sensors."
Other methods to represent functions in functional spaces include the use of Fourier coefficients \cite{li2020Fourier}, wavelets \cite{tripura2023wavelet}, spectral Chebychev basis \cite{fanaskov2023spectral}, reproducing kernel Hilbert spaces (RKHS) \cite{nelsen2021random}, graph neural operators \cite{li2020multipole} or meshless representations \cite{zhang2023belnet}.
Regarding the availability of data for the output function, we encounter two scenarios. In the first scenario, the functions in the output are known at the same fixed grid ${\bm{y}_1, \bm{y}_2,\dots,\bm{y}_{n}}$, where $y_i \in Y$; this case is termed as ``aligned" data. Conversely, there are cases where the output grid may vary randomly for each input function, known as ``unaligned" data. If this grid is uniformly sampled and dense enough, interpolation can be used to approximate the output function at fixed locations. Thus, this leads us back to the aligned data case. However, if the output is only available at sparse locations, interpolation becomes impractical. As we will see later in the text, despite this challenge, our approach can address this scenario, albeit with a higher computational cost for training the machine learning model (since, in such cases, the fixed structure of the data cannot be fully leveraged).

\section{Methods}
\label{sec:methods}
For the completeness of the presentation, we start with some preliminaries on the use of Neural Networks and DeepOnets for the approximation of nonlinear continuous functional operators. We then introduce the concepts of random projections for neural networks (RPNN) and extend them to the DeepOnet framework, arriving at the proposed Random Projection-based Operator networks (RandONets).

\subsection{Preliminaries on DeepOnets}
\label{sec:pre_deepOnets}
As universal approximators, feedforward neural networks (FNN) have the capability to approximate continuous functions effectively \cite{cybenko1989approximation,hornik1989multilayer,barron1993universal, leshno1993multilayer,pinkus1999approximation}. However, a lesser-known theorem by Chen \& Chen (1995) \cite{chen1995universal}, which gained prominence with the advent of DeepOnet by Lu et al. (2021) \cite{lu2021learning} and Fourier Neural Operator (FNO) by Li et al. (2020) \cite{li2020Fourier}, asserts the existence of a neural network architecture capable of approximating any continuous nonlinear operator to an arbitrary degree of accuracy. Before introducing this theorem, let us introduce the following definition:
\begin{definition}[Tauber-Wiener function \cite{chen1995universal}]
A function \( \psi \in \mathsf{C}(\R) \) is called a \( \mathsf{L}^p \) Tauber-Wiener function if, for any interval \([a,b] \subset \mathbb{R}\), the set of finite linear combinations
\[
\left\{ \sum_{i=1}^{N} w_i \psi(\xi_i x + \theta_i) \mid N \in \mathbb{N}, \xi_i \in \mathbb{R}, \theta_i \in \mathbb{R}, w_i \in \mathbb{C} \right\}
\]
is dense in \( \mathsf{L}^p([a,b]) \).
\end{definition}
\begin{theorem}[Universal approximation for functions\cite{chen1995universal}]
Suppose $K$ is a compact set in $\mathbb{R}^d$, $U$ is a compact set in $\mathsf{C}(K)$ and $\psi$ is a  Tauber-Wiener function, then $\forall$\( f \in \mathsf{U} \) and any \(\epsilon > 0\), there exist scaling factors \(\{\xi_i\}_{i=1}^{N}\) and shifts \(\{\theta_i\}_{i=1}^{N}\) both independent of $f$, and also coefficients \(\{w_i[f]\}_{i=1}^{N}\) depending on $f$, such that
\begin{equation}
    \left\| f(x) - \sum_{i=1}^{N} w_i[f] \psi(\xi_i x + \theta_i) \right\|_{\infty} < \epsilon.
\end{equation}
Moreover, the coefficient $w_i[f]$ are continuous functionals on $\mathsf{U}$.
\end{theorem}
In other words, any function in \( \mathsf{C}(K) \) can be approximated arbitrarily closely by a finite linear combination of scaled and shifted versions of \( \psi \). It has been demonstrated that continuous non-polynomial functions are Tauber-Wiener functions \cite{leshno1993multilayer,chen1995universal}.
Then, the following theorem holds:

\begin{theorem}[Universal approximation for operators \cite{chen1995universal}]
\label{thm:chen_chen}
    Suppose that $\psi$ is a Tauber-Wiener function, $\mathsf{X}$ is a Banach space, and $\mathsf{K}_1\subset \mathsf{X}$, $\mathsf{K}_2 \subset \R^d$ are two corresponding compact sets. Let $\mathsf{U}$ be a compact set in $\mathsf{C}(\mathsf{K}_1)$, and let $\mathcal{F}:\mathsf{U}\rightarrow\mathsf{C}(\mathsf{K}_2)$ be a nonlinear continuous operator. 
    Then, for any $\epsilon>0$, there are $N,M,m\in \mathbb{N}$, and network parameters $w_{ki}$, $\xi_{kij}$, $\theta_{ki}$, $\beta_{k} \in \R$, $\bm{c}_k \in \R^d$, $\bm{x}_j \in \mathsf{K}_1$, with $k=1,\dots,N$, $i=1,\dots,M$, $j=1,\dots, m$, such that:
    \begin{equation}
        \left| \mathcal{F}(u)(\bm{y}) - \sum_{k=1}^{N} \sum_{i=1}^{M} w_{ki} \psi \big( \sum_{j=1}^{m} \xi_{kij} u(\bm{x}_j) + \theta_{ki} \big) \cdot \psi({\bm{c}_k} \cdot \bm{y} + \beta_k) \right| < \epsilon, \quad \forall u \in U , \bm{y} \in  \mathsf{K}_2.
        \label{eq:chen_chen_network}
    \end{equation}
\end{theorem}
%
%

To briefly describe how the above Theorem in the original paper of Chen \& Chen in \cite{chen1995universal} works, let us assume that our goal is to approximate an operator $\mathcal{F}$, acting on the set of functions $u\in \mathsf{U}$. These functions $u$ (which are inputs to the DeepONet) are assumed to be known and sampled at $m$ fixed locations $x_j$ in the domain $\mathsf{K}_1$. The vector $\bm{U}=(u(\bm{x}_1),u(\bm{x}_2),\dots,u(\bm{x}_m)) \in \R^{m\times 1}$ (a column vector) is the input of a single-hidden layer FNN with $M$ neurons, the so-called \emph{branch network}, that process the function values space. At the same time there is a second single-hidden layer FNN with $N$ neurons, the so-called \emph{trunk network}, that process the new location $\bm{y} \in K_2 \subset \R^{1\times d}$ (for convenience let us assume it as a row vector) in which we have to evaluate the transformed function $\mathcal{F}[u]$. For convenience, let us define the vector $\bm{B}=(B_1,B_2,\dots,B_M)\in \R^{M\times 1}$ (column vector) of hidden layers value of the branch network:
\begin{equation}
    B_i(\bm{U})=\psi \big( \sum_{j=1}^{m} \xi_{kij} u(\bm{x}_j) + \theta_{ki} \big), i=1,2,\dots,M,
    \label{eq:branch}
\end{equation}
and let us define the vector $\bm{T}=(T_1,T_2,\dots,T_N) \in \R^{1\times N}$ (row vector):
\begin{equation}
    T_k(\bm{y})=\psi({\bm{c}_k}^T \cdot \bm{y} + \beta_k), k=1,2,\dots, N.
    \label{eq:trunk}
\end{equation}
Then, the output of the network as in Eq.~\eqref{eq:chen_chen_network} can be written as:
\begin{equation}
    \mathcal{F}[u](\bm{y})\simeq \sum_{k=1}^{N} \sum_{i=1}^{M} w_{ki} B_i(\bm{U}) T_k(\bm{y}) \Leftrightarrow  \mathcal{F}[u](y)=\bm{T} W \bm{B}=\langle \bm{T}, \bm{B}\rangle_W,
    \label{eq:branch_trunk}
\end{equation}
%
where the matrix $W \in \R^{N\times M}$ has elements $w_{ki}$. As can be seen, the output of the scheme is a weighted inner product $\langle\cdot,\cdot\rangle_{W}$ of the trunk and branch networks. In the next section, we will take advantage of this formulation for an efficient and accurate training of the network through the use of random bases.\par
We note, that the original theorem \ref{thm:chen_chen} considers only two shallow feedforward neural networks with a single hidden layer. On the other hand, DeepONet uses deep networks instead, but also can incorporate any other type of networks such as CNNs. An extension of the theorem \ref{thm:chen_chen}, given by Lu et al. \cite{lu2021learning}, states that the branch network and the trunk network can be chosen by diverse classes of ANNs, which satisfy the classical universal approximation theorem.
Also, while the Chen and Chen architecture in \eqref{eq:chen_chen_network} does not include an output bias, the DeepOnet usually utilize biases to improve generalization performance \cite{lu2021learning}. More broadly, DeepONets can be considered conditional models, where $\mathcal{F}[u](\bm{y})$ represents a function of $\bm{y}$ given $u$. These two independent inputs, $u$ and $\bm{y}$, are given as inputs to the trunk and branch networks, respectively. 
At the end, the embeddings of $u$ and $\bm{y}$ are combined through an inner product operation. However, the challenge remains in finding efficient approaches to train these networks.
%
It is also worth noting that, as it happens for shallow FNNs, while the universal approximation theorem for operators guarantees the existence of a successful approximation DeepOnet, it does not offer a numerical method for constructing the specific weights and biases of the networks. Furthermore, deep learning networks used in DeepONet do not come without limitations. While they enhance the models' ability to capture complex relationships, they introduce challenges in the optimization process. Specifically, determining the values of the networks \emph{parameters and hyperparameters} requires significant computational resources, entailing complexity that can lead to moderate generalization ability and/or numerical approximation accuracy.
Hence, it is nearly implicit that training such DeepOnet heavily relies on parallel computing and GPU-based computations.\par 
Here we present a computationally efficient method for approximating nonlinear operators, based on shallow networks with a single hidden layer, as in the paper of Chen and Chen in \cite{chen1995universal}, coupled with random projections, that relaxes the ``curse of dimensionality'' in the training process. 
First, we give some preliminaries for random projection neural networks, and then we proceed with the presentation of the RandONet and building on previous works, we prove its universal approximation property for linear and nonlinear operators.

\subsection{Preliminaries on Random Projection Neural Networks}
\label{sec:RPNN}
Random Projection Neural Networks (RPNN) are a type of single-hidden-layer ANNs with randomized activation functions to facilitate the training process. The family of RPNNs includes random weights neural networks (RWNN) \cite{schmidt1992feed}, Random Vector Functional Link Network (RVFLN) \cite{pao1992functional,igelnik1995stochastic}, Reservoir Computing (RC)\cite{jaeger2001echo,jaeger2002adaptive}, Extreme Learning Machines (ELM) \cite{huang2006extreme}
and Random Fourier Features Networks (RFFN) \cite{rahimi2007random,rahimi2008weighted}. Some seeds of this idea can be also found in \emph{gamba perceptron} proposed initially by Frank Rosenblatt (1962) \cite{rosenblatt1962perceptions} and the Distributed Method algorithm proposed by Gallant (1987) \cite{gallant1987random}. For a review on random projection neural networks see in \cite{scardapane2017randomness,fabiani2024random}.\par

Here we consider for simplicity, and refer with the acronym RP-FNN to, a single hidden layer feed-forward neural network, denoted by a vector function $\bm{f}\equiv(f_1,f_2,\dots,f_{n}):\R^d \rightarrow\R^{n}$, with $n$ outputs $f_k$, $N$ neurons and with an activation function $\psi:\R\rightarrow\R$.
To simplify our notation, we consider here each scalar output $y_k=f_k(\bm{x})$ of the RP-FNN:
\begin{equation}
f_k(\bm{x};W,\bm{\beta},C,\bm{P})=\sum_{j=1}^N w_{kj} \psi ({\bm{c}_j}\cdot \bm{x}+ \beta_j;\bm{P}),\qquad k=1,\dots,n,
\label{eq:RPNN}
\end{equation}
In RP-FNNs, the weights $W$ are the only trainable parameters of the network. While the internal weights/parameters and hyperparameters of $\psi_j$ are randomly pre-determined and fixed. In order to simplify the notation, let us group the set of parameters and hyperparameters in the vector of random variables $\bm{\alpha}$ over the set $ \mathsf{A}\subseteq\R^q$, containing the stacking of all parameters in $\{C,\bm{\beta},\bm{P}\}$. The vector of random variables $\bm{\alpha}$ is in general sampled from a probability distribution function $p_{\bm{\alpha}}$ on $\mathsf{A}$. Thus, we can rewrite Eq.~\eqref{eq:RPNN} as:
where the matrix $C \in \R^{N\times d}$, with rows $\bm{c}_j\in\R^{1\times d}$, contains \emph{a priori} randomly-fixed  \emph{internal weights}, sampled appropriately from a probability distribution $p_{\bm{c}}$, connecting the input layer with the hidden layer (in other configurations, they can also be set all to ones, see e.g.\cite{fabiani2023parsimonious}); the vector $\bm{\beta}=(\beta_1,\dots,\beta_N) \in \R^{N\times 1}$ includes \emph{a priori} randomly-fixed vector of \emph{biases} (shifts), sampled appropriately from a probability $p_{\beta}$; the vector $\bm{x}\in\R^{d\times 1}$ represents the input, the matrix $W\in\R^{n\times N}$, with elements $w_{kj}$, contains the \emph{external} weights that connect the hidden layer to the output; the vector $\bm{P}$ includes any additional required hyperparameters (either deterministically or randomly fixed) for the activation function, for example the shape parameters of Gaussian Radial Basis functions.\par
\begin{equation}
    f_k(\bm{x};\bm{\alpha})=\sum_{j=1}^N w_{kj} \psi(\bm{x};\bm{\alpha}_j),
\end{equation}
where $\bm{\alpha}_j\in\mathsf{A}$ are $N$ realizations of the random variables $\bm{\alpha}$.\par
Having  $\{\bm{x}_i, f_k(\ \bm{x}_i)\}, \, i=1,2,\dots,m$ pairs of training data for each $f_k$, the unknown parameters $\bm{w}_k\in\R^N$, which are the rows of the matrix $W$, can be computed for example using truncated SVD, preconditioned QR decomposition with regularization, and/or Tikhonov regularization. For example, with truncated SVD, the least-squares solution reads:
\begin{equation}
\bm{w}_k = \sum_{i=1}^k \frac{{u_i}^T Y_k}{\sigma_i}v_i,
\end{equation}
with $u_i$, $v_i$ being the first $k$ right and left singular vectors of $\Psi$ is the $N\times m$ matrix with entries $\Psi_{ji}=\psi (\bm{x}_i;\bm{\alpha}_j)$, and $\sigma_i$, the corresponding singular values.
On the other hand, the Tikhonov regularization reads:
\begin{equation}
\bm{w}_k = \arg \min_{\bm{w}_k} \left\{ \|\bm{w}_k \Psi - Y_k\|^2 + \lambda \|\bm{w}_k L\|^2 \right\},
\end{equation}
where  $Y_k$ is the vector of dimension $m$, containing the values (samples) of $f_k$ at $m$ sampling points $\bm{x}_i$, $\lambda > 0$ is the regularization parameter and $L\in\R^{N\times N}$ is a regularization operator, often taken as the identity matrix $I$.
The Tikhonov regularized solution can be expressed as:
\begin{equation}
\bm{w}_k = Y_k \Psi^T (\Psi \Psi^T + \lambda L L^T)^{-1}.
\label{eq:Tregular}
\end{equation}
Setting, $L=I$, the above problem can be solved, e.g., by substituting the truncated SVD of \(\Psi\) into the Tikhonov regularization formula to get: \cite{fierro1997regularization}:
\begin{equation}
\bm{w}_k = \sum_{i=1}^r \frac{{\sigma_i}^2}{{\sigma_i}^2+\lambda^2}\frac{{u_i}^T Y_k}{\sigma_i}v_i,
\end{equation}
where $r$ is the rank of $\Psi$.
Let us now define the \emph{hidden layer map} $\bm{\phi}_N:\R^d\times\mathsf{A}\rightarrow\R^N$, that maps the input layer to the (random) features $\bm{z}$ of the hidden layer, as:
\begin{equation}
    \bm{z}=\bm{\phi}_N(\bm{x};\bm{\alpha})=[\psi({\bm{c}_{1}} \cdot \bm{x}+\beta_1;P),\psi({\bm{c}_{2}} \cdot \bm{x}+\beta_2;P),\dots,\psi({\bm{c}_{N}} \cdot \bm{x}+\beta_N;P)].
\end{equation}
In its simplest form (taking just a linear projection of the input space), the above is -- conceptually equivalent -- to the celebrated Johnson-Lindenstrauss (JL) lemma \cite{johnson1984extensions}, which states that there exists an approximate isometry projection $\bm{\phi}_N: \R^d \rightarrow \R^N$ of input data $\bm{x}\in\R^{d}$ induced by a random matrix $R \in \R^{N \times d}$:
\begin{equation}\label{JL}
\bm{z}^{JL}=\bm{\phi}^{JL}_N (\bm{x};R) = \dfrac{1}{\sqrt{N}} R \,\bm{x},
\end{equation}
where the matrix $R = [R_{ji}] \in \R^{N \times d}$ has components which are i.i.d.~random variables sampled from a standard normal distribution. Let us assume $X$ to be a set of $m$ sample points $\bm{x}\in \R^d$, such that $N\ge O\big(\log (m)/\epsilon^2\big)$. Then with probability $\mathbb{P}$, for $\forall \epsilon \in (0,1)$, we obtain:
\begin{equation}
    \mathbb{P}\big(\big|\|\bm{x}\|_2-\|\bm{z}^{JL}\|_2\big|\le \epsilon\|\bm{x}\|_2\big)\ge 1-2\exp\biggl(-(\epsilon^2-\epsilon^3)\frac{N}{4}\biggr).
\end{equation}
Let us consider a simple regression problem, with training data $(\bm{x}^{(s)},\bm{y}^{(s)})\in\R^d\times \R^n$, $s=1,\dots,m$. Let us call the matrix $X\in\R^{d\times m}$ the collection of inputs, and the matrix $Y\in\R^{n\times m}$ the collection of outputs. A simple approach can consist in considering linear random JL projections of the input, and then approximating the output as a weighted linear combination of random JL features. Thus, one finds $W \in \R^{n\times N}$ such that:
\begin{equation}
    Y= \frac{1}{\sqrt{N}} W  R X.
\end{equation}
At this stage, one can solve for the unknown parameters in $W$ using a linear regularization problem as briefly described above.
When considering nonlinear projections, the training of an RP-FNN, 
involves solving a system of $n\times m$ linear algebraic equations, with $n \times N$ unknowns:
\begin{equation}
    W\Phi_N(X;\bm{\alpha})=Y, \qquad \Phi_{ji}=\psi(\bm{x}_i;\bm{\alpha}_j),
    \label{eq:RPNN_solve}
\end{equation}
where $\Phi_N(X) \in \R^{N\times m}$ is the random matrix of the hidden layer features, with elements $\Phi_{ji}$. Note that despite the nonlinearity of $\psi$, the problem is still linear with respect to the external weights $W$.\par 
While JL linear random projections are appealing due to their simplicity, studies have highlighted that well-designed \emph{nonlinear} random projections can outperform such linear embeddings \cite{barron1993universal, igelnik1995stochastic, gorban2016approximation,fabiani2024random}. 
In this context, back in '90s Barron \cite{barron1993universal} proved that for functions with integrable Fourier transformations, a random sample of the parameters of sigmoidal basis functions from an appropriately chosen distribution results to an approximation error on the order of $O(1/(n^{(2/d)}))$. Igelnik and Pao \cite{igelnik1995stochastic} extended Barron's proof \cite{barron1993universal} for any family of $L^2$ integrable basis functions, thus addressing the so-called RVFLNs \cite{pao1992functional}. Later on, the works of Rahimi and Recht \cite{rahimi2007random,rahimi2008weighted} have explored the effectiveness of nonlinear random bases for preserving any shift-invariant kernel distances. It is also worth mentioning that the ``kernel trick" \cite{scholkopf2000kernel,scholkopf2005kernel}, a common feature approach in machine learning, including Support Vector Machines (SVMs) and Gaussian Processes (GPs), provides a straightforward method to generate features for algorithms that rely solely on the inner product between pairs of input points: 
\begin{equation}
    \langle \tilde{\bm{\phi}}(\bm{u}), \tilde{\bm{\phi}}(\bm{v}) \rangle = K(\bm{u},\bm{v}),
\end{equation}
where $\tilde{\bm{\phi}}$ represents a generic implicit lifting and $K$ is a kernel distance function. This technique is commonly employed to effectively handle high-dimensional data without explicitly computing the feature vectors $\tilde{\bm{\phi}}(\bm{u})$ and $\tilde{\bm{\phi}}(\bm{v})$.\par
However, large training sets often lead to significant computational and storage costs. Instead of depending on the implicit lifting provided by the kernel trick, we seek an \emph{explicit} mapping of the data to a low-dimensional Euclidean inner product space, using a nonlinear randomized (randomly parametrized) feature map $\bm{\phi}_N(\cdot,\bm{\alpha}): \mathbb{R}^d \times \mathsf{A}\rightarrow \mathbb{R}^N$:
\begin{equation}
    K(\bm{u},\bm{v}) \approx \bm{\phi}_N(u;\bm{\alpha})^T \bm{\phi}_N(v;\bm{\alpha}),
\end{equation}
where $\bm{\alpha}\in\R^q$ is a set of hyperparameters (random variables) sampled from a probability distribution $p_{\bm{\alpha}}$.\par
Here, for the sake of completeness of the presentation, we restate the following theorem \cite{rahimi2007random},
\begin{theorem}[Low-distortion of kernel-embedding \cite{rahimi2007random}]
Let $K$ be a positive definite shift-invariant kernel $K(\bm{u}, \bm{v}) = K(\bm{u} - \bm{v})$. Consider the Fourier transform $p_{K,\bm{\alpha}}=\hat{\mathcal{F}}[K]$ of the kernel $K$, resulting a probability density function (pdf) $p_{K,\bm{\alpha}}$ in the frequency space $\mathsf{A}$: $p_{K,\bm{\alpha}} (\alpha) = \frac{1}{2\pi} \int e^{j\alpha\Delta}K(\Delta)d\Delta$, and draw $N$ i.i.d.~samples weights $\bm{\alpha}_1, \dots, \bm{\alpha}_N \in\mathbb{R}^d$ from $p_{K,\bm{\alpha}}$. Define
\begin{equation}
    \bm{\phi}_N(\bm{u};\bm{\alpha}) \equiv \sqrt{\frac{1}{N}} [\cos(\bm{\alpha}_1^T \bm{u}), \dots, \cos(\bm{\alpha}_n^T \bm{u}),\sin(\bm{\alpha}_1^T \bm{u}),\dots,\sin(\bm{\alpha}_n^T \bm{u})].
\end{equation}
Then,$\forall \epsilon>0$
    \begin{equation}
        \mathbb{P}\left[(1-\epsilon)K(\bm{u},\bm{v}) \leq \bm{\phi}_N(\bm{u})^T \bm{\phi}_N(\bm{v})\leq (1+\epsilon)K(\bm{u},\bm{v})\right] \geq 1-O\left(\exp\left(-\frac{N\epsilon^2}{4(d+2)}\right)\right),
    \end{equation}
    where $\mathbb{P}$ stands for the probability function. 
\end{theorem}
The above approach for the kernel approximation, employing trigonometric activation functions, is also known as Random Fourier Features (RFFN). An equivalent result can be obtained by employing only cosine as the activation function and random biases in $[0,2\pi]$ \cite{rahimi2007random}. More generally, there is no constraint in considering trigonometric activation functions, as sigmoid and radial basis functions have equivalently shown remarkable results \cite{fabiani2021numerical,fabiani2023parsimonious,fabiani2024random,dong2021computing,dong2021local,dietrich2023learning,patsatzis2023slow}. Here we restate, the following theorem \cite{rahimi2008weighted,rahimi2008uniform}:
\begin{theorem}(cf. Theorem 3.1 and 3.2 in \cite{rahimi2008uniform}))
\label{thm:rahimi}
Consider the parametric set activation functions on $\mathsf{X}\subseteq\R^d$, $\psi(\bm{x};\bm{\alpha}):\mathsf{X}\times \mathsf{A} \rightarrow \R$ parametrized by random variables $\bm{\alpha}$ in $\mathsf{A}$, that satisfy $\sup_{\bm{x},\bm{\alpha}}|\phi(\bm{x},\bm{\alpha})|\leq 1$. Let $p$ be a probability distribution on $\mathsf{A}$ and $\mu$ be a probability measure on $\mathsf{X}$ and the corresponding norm $\|f \|_{L^2(\mu)}=\int_{\mathsf{X}} f(\xb)^2\mu(d\xb)$. Define the set:
\begin{equation}\label{infsolution}
\mathsf{G}_p \equiv \biggl\{g(\xb)=\int_{\mathsf{A}} w(\bm{\alpha})\phi(\bm{x};\bm{\alpha}) d\bm{\alpha} : \, \|g\|_{p(\bm{\alpha})}<\infty \biggr\},\, \| g \|_p:=\sup_{\bm{\alpha} \in \mathsf{A}}\|w(\bm{\alpha})/p(\bm{\alpha}) \|.
\end{equation}
Fix a function $g^*$ in $\mathsf{G}_p$. Then, for any $\delta >0$, there exist $N\in\N$, and $\bm{\alpha}_1,\bm{\alpha}_2,\dots,\bm{\alpha}_N$ of $\bm{\alpha}$ drawn i.i.d. from $p$, and a function $\hat{g}$  in the random set of finite sums
\begin{equation}
\mathsf{\hat{G}}_{\bm{\alpha}}\equiv \biggl\{\hat{g} \, : \, \hat{g}(\xb)=\sum_{j=1}^N w_j \phi(\bm{x};\bm{\alpha}_j)\biggr\}
\end{equation}
such that
\begin{equation}\label{uniformconv}
\sqrt{\int_{\mathsf{X}}(g^*(\bm{x})- \hat{g}(\bm{x}))^2d\bm{\mu}(\bm{x})}\leq\dfrac{||g^*||_p}{\sqrt{N}}\biggl(1+\sqrt{2\log \dfrac{1}{\delta}}\biggr),
\end{equation}
holds with probability at least $1-\delta$.
Moreover, if $\phi(\bm{x} ;\bm{\alpha})=\varphi(\bm{\alpha}\cdot\bm{x})$, for a $L$-Lipschitz function $\psi$, the above approximation is uniform (i.e. in the supremum norm).
\end{theorem}
The above Theorem implies that the function class $G_p$ can be approximated to any accuracy when $N\rightarrow \infty$. Moreover (see \cite{rahimi2008uniform}) this class of functions is dense in Reproducing Kernel Hilbert Spaces defined by $\phi$ and $p$.
For a detailed discussion on the pros and cons of function approximation with such random bases see \cite{gorban2016approximation}. 
Finally, very recently, Fabiani (2024) \cite{fabiani2024random} have theoretically proved the existence and uniqueness of RP-FNN of the best approximation and their exponential convergence rate when approximating low-dimensional infinitely differentiable functions. These theoretical results are also numerically validated through extensive benchmarks in \cite{fabiani2024random}. This showcases a concrete possibility of bridging the gap between theory and practice in ANN-based approximation \cite{adcock2021gap,fabiani2024random}.

\subsection{Random Projection-based Operator Networks (RandONets)}
\label{sec:RandONets}
In this section, we present RandONets for approximating nonlinear operators. Building on previous works, we first demonstrate that the proposed shallow--single hidden layer--random projection neural networks are universal approximators of non-linear operators. Then we discuss how RandONets can be used in both the aligned and unaligned data cases, and finally present their numerical implementation for the solution of the inverse problem: the approximation of the differential operator, i.e., the RHS of the differential equations as well as their solution operator. 
\begin{figure}[ht!]
    \centering
    \includegraphics[trim={4.5cm 0cm 2cm 1cm},clip,width=0.6\textwidth]{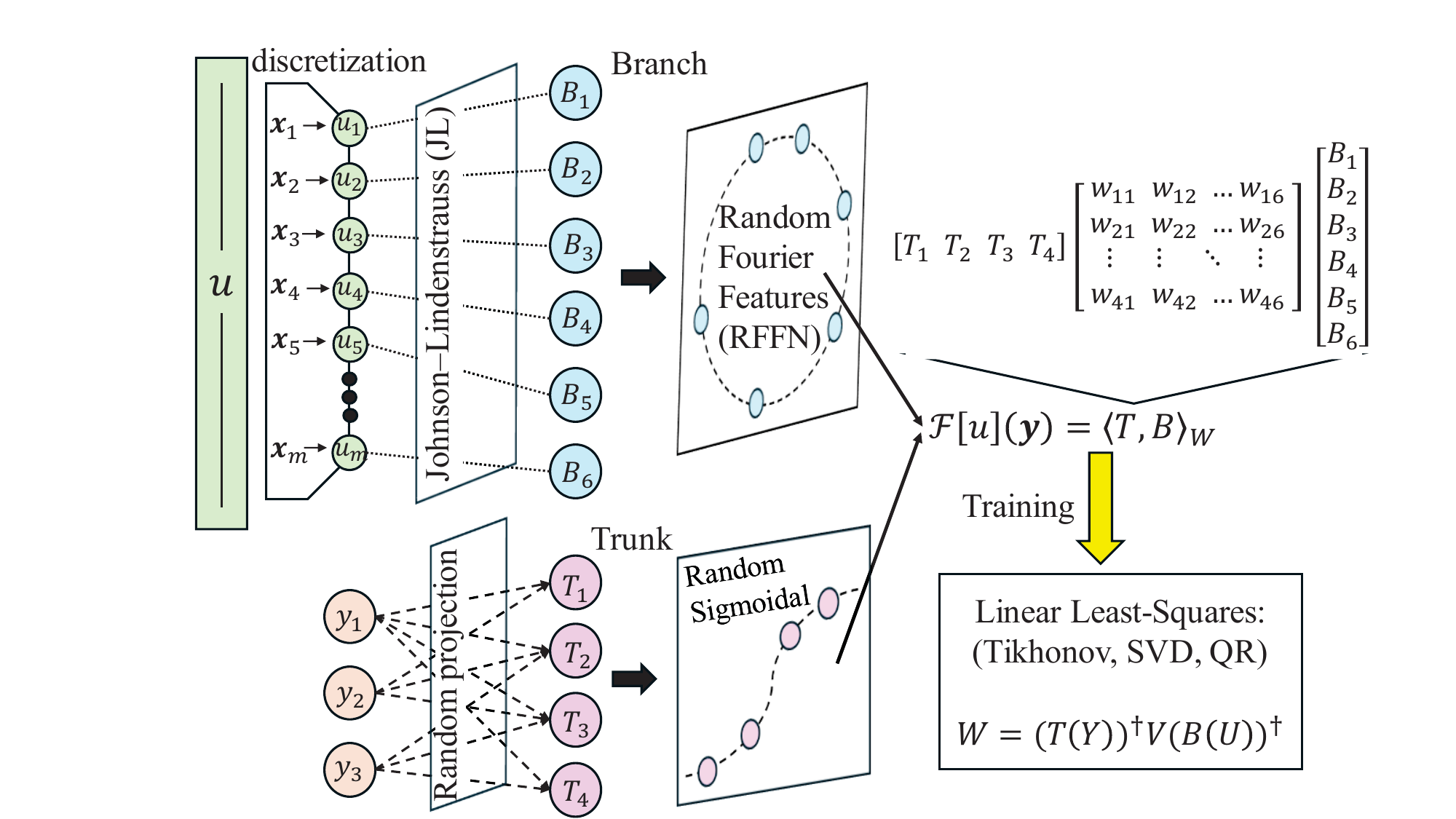}
    \caption{Schematic of the Random Projection-based Operator Network (RandOnet). The RandOnet first discretizes the input function ($u$) over a fixed grid of spatial points. Then it separately embeds the space of the spatial locations ($\bm{y}$) into a random hidden layer (e.g., with sigmoidal activations functions) and the space of the discretized functions into low-distortion kernel-embedding (e.g., with Johnson-Lindenstrauss random projections \cite{johnson1984extensions} or Rahimi and Recht Random Fourier Features \cite{rahimi2007random}). Finally, the output is composed of a weighted ($W$) inner product of the branch ($B$) and trunk ($T$) features. The training can be performed through linear least-squares techniques (e.g., Tikhonov regularization, SVD and QR decomposition).}
    \label{fig:randOnet}
\end{figure}

\subsubsection{RandONets as universal approximators of nonlinear operators}
\label{sec:theory}
In this section, we prove that RandONets are universal approximators of nonlinear operators.
Following the methodological thread in \cite{chen1995universal}, we first state the following proposition:

\begin{prop}
\label{thm:theorem3chenchen_randonet}
Let $\mathsf{K} \subset \R^d$ compact and $\mathsf{U} \subset \mathsf{C}(\mathsf{K})$ compact and consider a parametric family of random activation functions $\{ \psi(\bm{x} ; \bm{\alpha}) \,\, : \bm{x} \in \R^{d}, \,\, \bm{\alpha} \in \mathsf{A}\}$, where $\bm{\alpha} \in \mathsf{A}$ is a vector of randomly chosen (hyper) parameters, and assume that $\psi$ are uniformly bounded in $\R^d \times \mathsf{A}$. Let $p$ be a probability distribution on $\mathsf{A}$. Given any $\epsilon$, there exists a $N \in {\mathbb N}$ and i.i.d. sample $\bm{\alpha}_1, \cdots, \bm{\alpha}_N$ from $p$, chosen independently of $f$, such that for every $f \in \mathsf{U}$ the random approximation
\begin{equation}
f_{\epsilon}(\bm{x})= \sum_{j=1}^{N} c_{j}[f] \psi(\bm{x} ; \bm{\alpha}_j),
\end{equation}
approximates $f$ in the sense that with high probability
\begin{equation}
\| f - f_{\epsilon} \|_{L^2(\mu)} < \epsilon,
\end{equation}
for a suitable probability measure $\mu$ over $\mathsf{K}$. 
Moreover, if $\psi(\bm{x} ;\bm{\alpha})=\varphi(\bm{\alpha}\cdot\bm{x})$, for a $L$-Lipschitz function $\psi$, the above approximation is uniform (i.e. in the supremum norm).
\end{prop}

\begin{proof} We only show the uniform approximation result. The $L^2(\mu)$ approximation follows similar arguments.

We apply Theorem 3 in \cite{chen1995universal}.  Given a  continuous sigmoidal  (non polynomial) function $\sigma$,   for every $\epsilon >0$, there exist $N \in {\mathbb N}$, $(\hat{\theta}_i, \bm{\omega}_i) \in \R \times \R^{d}$, $i=1, \cdots, N$ such that for every $f \in \mathsf{U} \subset \mathsf{C}(\mathsf{K})$, there exist a linear continuous functional on $U$, $f \mapsto c_{i}[f]$, with the property that
\begin{equation}\label{PAG-RR-1}
|f(\xb)- \sum_{i=1}^{N} c_{i}[f] \sigma(\bm{\omega}_i \cdot \bm{x} + \hat{\theta}_{i} ) | < \epsilon, \,\,\,\, \forall \, \xb \in K.
\end{equation}
This approximation is deterministic. Note that even though according to Theorem 3 op. cit. this approximation holds for any Tauber-Wiener function (i.e. even for non-continuous $\sigma$) here we must insist on the continuity of $\sigma$.

In the above, we obtain a uniform approximation $f_{\epsilon}$ to $f$ in terms of
\begin{equation}\label{PAG-RR-2}
f_{\epsilon}(\xb)=\sum_{i=1}^{N} c_{i}[f] \sigma(\bm{\omega}_i \cdot \bm{x} + \hat{\theta}_{i} )
\end{equation}
We emphasize the $\sigma$ is chosen to be a continuous  and bounded sigmoidal function. Note that the choice of $N$, $(\theta_i, \omega_i)$, are independent of $f$, while the coefficients $c_{i}[f]$ depend on $f$ and in fact $f \mapsto c_{i}[f]$ is a linear continuous functional on $U$.

To connect with the random features approach of Rahimi and Rechts we  first employ an expansion of each sigmoid  in \eqref{PAG-RR-1} in terms of a radial basis function (RBF) (which is eligible to a random features expansion a la Rahimi and Rechts) and then follow with the expansion of the RBFs in random features.  We follow the subsequent steps:

(a) For each $i=1, \cdots, N$, we consider the function $\phi _i$, defined by $ \mathsf{K} \ni \bm{x} \mapsto \phi_i(\bm{x};\bm{\omega}_i,\hat{\theta}_i):=\sigma( \bm{\omega}_i \cdot \bm{x} + \hat{\theta}_i)$.  By the properties of $\sigma$, the functions $\phi_i \in \mathsf{C}(\mathsf{K})$.
Hence, we may apply the approximation of each $\phi_i$ in terms of an RBF neural network.  Using standard results (e.g. Theorem 3 \cite{liao2003relaxed}) 
we have that if $g \in \mathsf{C}(\R^d)$ is a  bounded radial basis function
${\cal S}:={\rm span}\{ g(a \xb + \bm{b}) \,\, : \,\, a \in \R, \,\, \bm{b} \in \R^{d}\}$ is dense in $\mathsf{C}(\mathsf{K})$.  Note that without loss of generality we may impose the extra assumption that $g$  can be  expressed in terms of the inverse Fourier transform of some function (i.e. an element of a function space on which the Fourier transform is surjective, for example, $g$ belongs in the Schwartz space). This assumption  also allows us to invoke the standard results of  \cite{Park1991universal}, \cite{park1993approximation} leading to the same density result). Note that this step does not affect the generality of our results, as it is only used in the intermediate step (a) which re-expands the general sigmoids used in \eqref{PAG-RR-1} into a more convenient basis on which the step (b) is applicable. Moreover, the choice of $g$ is not unique.

Using the above result we can approximate each $\phi_i$ in terms of the functions
\begin{equation}\label{PAG-RR-10}
\phi_{i,\epsilon}(\xb) = \sum_{j=1}^{M_{i}} w_{ij} g(a_{ij}  \bm{x} + \bm{b}_{ij}),
\end{equation}
where importantly, the $(w_{ij}, a_{ij} , \bm{b}_{ij}) \in \R \times \R \times \R^d$ are independent of the choice of the function $f$ (as they only depend on $(\bm{\omega}_i, \hat{\theta}_i) \in \R^{d} \times \R$, which are independent of $f$ -- see \eqref{PAG-RR-1}). The function $\phi_{i,\epsilon}$ satisfies the property $\| \phi_i - \psi_{i,\epsilon}\| < \frac{\epsilon}{N}$, in the uniform norm $\forall \epsilon>0$. Combining \eqref{PAG-RR-1} (and \eqref{PAG-RR-2}) with \eqref{PAG-RR-10}, we obtain an approximation $f^{RBF}_{\epsilon}$ to $f_{\epsilon}$ in terms of
\begin{equation}\label{PAG-RR-11}
\begin{aligned}
f^{RBF}_{\epsilon}(\xb) = \sum_{i=1}^{N} \sum_{j=1}^{M_i} c_{i}[f] w_{ij} g(a_{ij} \bm{x} + \bm{b}_{ij}),
\end{aligned}
\end{equation}
such that $\| f_{\epsilon} - f^{RBF}_{\epsilon} \| < \epsilon$, hence satisfying by \eqref{PAG-RR-1} that $\| f - f_{\epsilon}^{RBF} \| < 2 \epsilon$ (in the uniform norm).

(b) Now,  by appropriate choice of the RBF function $g$ we apply Rahimi and Recht for a further expansion of each term 
$g(a_{ij} \xb + \bm{b}_{ij}) $, using the random features RPNN. By the choice of $g$ as above,  this is now possible, and the results  of Rahimi and Recht \cite{rahimi2008uniform}(theorem \ref{thm:rahimi}), are now applicable to $g$. This holds, since RBFs can be elements of the RKHS that the RR framework applies to, i.e.,  they belong to the space of functions ${\cal G}_p$, defined in (\ref{infsolution}).
For such a choice Theorem 3.1 in \cite{rahimi2008uniform} can be directly applied to each of the RBF $g$ in the function \eqref{PAG-RR-11}. 
Under the extra assumption that $\phi(\xb ; \bm{\alpha})=\varphi(\bm{\alpha} \cdot \xb)$, for $\varphi$ L-Lipschitz (which without loss of generality can be shifted so that $\varphi(0)=0$ and scaled so that $\sup |\varphi | \le 1$), we can also apply Theorem 3.2 op. cit for a corresponding uniform approximation. We only present the second case, the first one being similar. There are two equivalent ways to proceed.

b1) Using Theorem 3.2 op cit, for an $L$-Lipschitz $\varphi$ as defined above, for any $\delta > 0$ there exists a random function $g_{\delta}$ of the form
\begin{equation}\label{PAG-RR-20}
g_{\delta}(\xb) = \sum_{k=1}^{K} \hat{c}_{k} \varphi(\bm{\alpha}_{k} \cdot \bm{x}),
\end{equation}
where $\bm{\alpha}_{k}$ is i.i.d. randomly sampled from a chosen distribution $p$, which approximates $\hat{\epsilon}(\delta)$ close $g$ with probability at least $1-\delta$. 

Using \eqref{PAG-RR-20} into \eqref{PAG-RR-11} we obtain
\begin{equation}\label{PAG-RR-111}
\begin{aligned}
f_{\epsilon, \delta}(\xb) = \sum_{i=1}^{N} \sum_{j=1}^{M_i} \sum_{k=1}^{K} c_{i}[f] w_{ij} \hat{c}_{k} \varphi(\bm{\alpha}_{k}\cdot (a_{ij} \bm{x} + \bm{b}_{ij})),
\end{aligned}
\end{equation}
which if $\delta$ is chosen such that $\hat{\epsilon}(\delta) < \frac{\epsilon}{N M}$ satisfies $\| f^{RBF}- f_{\epsilon, \delta} \| < \epsilon$, hence, $\| f -f_{\epsilon, \delta} \| < 3 \epsilon$.

Using a resummation of \eqref{PAG-RR-111} in terms of a single summation index $\ell$
we end up with an approximation $f_{\epsilon, \delta}$ for $f$ in the form 
\begin{equation}\label{PAG-RR-300}
f_{\epsilon, \delta}(\xb)=\sum_{\ell=1}^{\hat{N}} \hat{w}_{\ell}[f] \varphi(\bm{\alpha}_{\ell}\cdot (a_{\ell} \bm{x} + \bm{b}_{\ell})),
\end{equation}
where $\hat{w}_{\ell}[f]=c_{i}[f] w_{ij} \hat{c}_{k}$ are continuous functionals on $\mathsf{U}$. Hemce, we obtain an approximation in terms of shifted and re-scaled $L$-Lipschitz random feature functions.

b2) One possible drawback of this expansion is that -- see \eqref{PAG-RR-111} it depends both on the $a_j$, $\bm{b}_j$ and the $\bm{\alpha}_{k}$ and not on the $\bm{\alpha}_k$ only.  An alternative could be to expand each one of the $g_j(\xb):=g(\alpha_j \xb + \bm{b}_j)$ separately. If it holds that $g_j \in {\cal G}$ for each $j$ then 
\begin{equation}
g_{j}(\xb)=\sum_{k=1}^{K} \hat{c}_{jk} \varphi(\bm{\alpha}_{k}^{(j)} \cdot \bm{x}), \,\,\, j=1, \cdots, M,
\end{equation}
where $\bm{\alpha}^{(j)} := \{\bm{\alpha}_{k}^{(j)}, \,\, : \,\, k=1, \cdots, K\} \sim_{i.i.d} p$ for each $j=1, \cdots, M$ and with the $ \hat{\bm{\alpha}}^{(j)}$ for different $j$ being independent.
When using this approach we get the expansion \eqref{PAG-RR-300} with ${\alpha}_{k}$ i.i.d. from our initial distribution $p$. Upon resummation the stated result follows.
\end{proof}

Based on the proposition \ref{thm:theorem3chenchen_randonet}, we can now prove the following proposition for universal approximation of functional $\mathcal{F}:U\to \R$ in terms of the random projection neural network (RPNN):
\begin{prop}
[Random Projection Neural Networks (RPNNs) for functionals]
\label{thm:chen_chen4_randonet_functionals}
Adopting the framework from Proposition 1, and additionally, let
    $\mathsf{U}$ be a compact subset of $\mathsf{C}(\mathsf{K})$ and $\mathcal{F}$ be a continuous functional in $U$. Let us define the compact set $\mathsf{U}_m \subseteq \R^d$ of vectors, whose elements consist of the values of the function $u \in \mathsf{U}$ on a finite set of $m$ grid points $\bm{x}_1,\dots,\bm{x}_m \in \R^d$ and denote the vector $\bm{u}:=[u(\bm{x}_1),\dots,u(\bm{x}_m)] \in \R^m$. 

Then,  
with high probability, w.r.t. $p$, for any $\epsilon>0$, there exist $M,m\in\N$, $\bm{\alpha}_1,\dots,\bm{\alpha}_M \in \mathsf{A}$, i.i.d distributed from $p$, such that:
\begin{equation}
  \left\| \mathcal{F}(u) - \sum_{i=1}^{M} w_{i} \varphi \left( \bm{\alpha}_i\cdot \bm{u}(\bm{x}) \right) \right\|_{\infty} < \epsilon, \quad \forall u \in U.
\end{equation}
\end{prop}
\begin{proof}
The representation of the function $u\in\mathsf{U}$ through a finite set of $m$ evaluations $u(\bm{x}_j)$ is possible by the Tietze Extension Theorem for functionals from the set $\mathsf{U}_m$ to $\mathsf{U}$ (see \cite{chen1995universal} for more details). Then the proof comes directly from Proposition \ref{thm:theorem3chenchen_randonet}.
\end{proof}
Finally, using the above ideas and results and possibly allowing for different randpom embeddings for the branch and trunk networks, we can prove the following theorem:
\begin{theorem}[RandONet universal approximation for Operators]
\label{thm:RandONet_universal_chenchen5}
Adopting the framework of propositions 1,2 and the notation of Theorem 3.2, and additionally, let:
$\mathsf{X}$ be a Banach Space, and $\mathsf{K}_1\subset \mathsf{X}$, $\mathsf{K}_2 \subset \R^d, \mathsf{U}\subset C(\mathsf{K}_1)$ be compact sets, and $\mathcal{F}:\mathsf{U}\rightarrow\mathsf{C}(\mathsf{K}_2)$ be a continuous (in the general case nonlinear) operator. 
Then, 
with high probability w.r.t. $p$, for any $ \epsilon>0$, there exist positive integers $M, N, m \in \mathbb{N}$, and  network (hyper)parameters $\bm{\alpha}^{br,tr}_1,\dots,\bm{\alpha}^{br,tr}_N \in \mathsf{A}^{br,tr}$, i.i.d distributed from $p_{\bm{\alpha}}$ 
such that:
\begin{equation}
\left\| \mathcal{F}(u)(\bm{y}) - \sum_{k=1}^{N} \sum_{i=1}^{M} w_{ki} \varphi^{br} \left(\bm{\alpha}^{br}_i \cdot \bm{u}(\bm{x})\right) \varphi^{tr}(\bm{\alpha}^{tr}_k\cdot \bm{y}) \right\|_{\infty} < \epsilon, \quad \forall u \in U , \bm{y} \in  \mathsf{K}_2,
\label{eq:chen_chen_network2}
\end{equation}
where the superscripts $br,tr$ correspond to branch and trunk networks and can be chosen in generally independently.
\end{theorem}
\begin{proof}
From the Proposition \ref{thm:theorem3chenchen_randonet}, we have that 
with high probability, for any ${\epsilon_1>0}$, 
there are $N\in\N$, $\tilde{w}_k[\mathcal{F}[u]]$ and $\bm{\alpha}^{tr}_k \in \mathsf{A}^{tr}$, such that
\begin{equation}
    \left\| \mathcal{F}(u)(\bm{y}) - \sum_{k=1}^{N} \tilde{w}_k[\mathcal{F}[u]] \varphi^{tr}(\bm{\alpha}^{tr} \cdot \bm{y}) \right\|_{\infty} < {}{\epsilon_1}.
    \label{eq:thm5_eq1}
\end{equation}
Moreover, from Proposition \ref{thm:theorem3chenchen_randonet}, we have that for any $k=1,\dots,N$, $\tilde{w}_k[\mathcal{F}[u]]$ is a continuous functional on $\mathsf{U}$. We can therefore repeatedly apply Proposition \ref{thm:chen_chen4_randonet_functionals}, for each $k$, and obtain approximations of each functional $\tilde{w}_k[\mathcal{F}[u]]$ on $\mathsf{U}_m$. Thus, with high probability, for any $\epsilon_{2}>0$ there exist $m,M\in\N$, $w_{ki}$, $\bm{\alpha}^{br}_{i} \in \mathsf{A}^{br}$, such that:
\begin{equation}
    \biggl\| \tilde{w}_k[\mathcal{F}[u]] - \sum_{i=1}^M w_{ik} \varphi^{br}\left( \bm{\alpha}^{br}_i\cdot \bm{u}(\bm{x})\right)  \biggr\|_{\infty} < \epsilon_{2}.
    \label{eq:thm5_eq2}
\end{equation}
Combining \eqref{eq:thm5_eq1} and \eqref{eq:thm5_eq2} 
we obtain Eq.~\eqref{eq:chen_chen_network2}. This completes the proof.
\end{proof}

\subsubsection{Implementation of RandONets}
\label{sec:RandONets_implementation}
In this section, we present the architecture of RandONets, depicted in Figure \ref{fig:randOnet}. As in Theorem \ref{thm:chen_chen}, we use two single-hidden-layer FNNs with appropriate random bases as activation functions. We employ (nonlinear) random based projections for embedding, in the two separate hidden layer features, both the (high-dimensional) space of the discretized function ($u$) and the domain (low-dimensional) of spatial locations ($\bm{y}$) of the transformed output ($v(\bm{y})=\mathcal{F}[u](\bm{y})$). 

Specifically, we propose leveraging nonlinear random projections to embed the space of spatial locations efficiently, employing parsimoniously chosen random bases. Thus, the random projected-based trunk feature vector $\bm{T}=(T_1,\dots,T_N)$, as denoted in Eq.~\eqref{eq:trunk}, can be re-written as:
\begin{equation}
\bm{T}=\bm{\varphi}^{tr}_n(\bm{y};\bm{\alpha}^{tr})=[\varphi( \bm{y} \cdot \bm{\alpha}^{tr}_1 +b_1),\dots,\varphi( \bm{y}\cdot \bm{\alpha}^{tr}_N +b_N)],
\label{eq:trunk_embedding}
\end{equation}
where $\bm{\alpha}^{tr}_k\in\R^d$, $b_k$, $j=1,\dots,N$ are i.i.d. randomly sampled from a continuous probability distribution function.

Here, when the domain of $\mathcal{F}[u]$ is a one-dimensional interval $[a,b]\subseteq\R$, we select the activation function $\varphi$ of the trunk network to be the hyperbolic tangent, and we utilize a parsimonious function-agnostic randomization of the weights as explained in \cite{fabiani2024random,fabiani2021numerical,fabiani2023parsimonious}. In particular, the weights $\bm{\alpha}^{tr}_j$ are uniformly distributed as $\mathcal{U}[-a_U,a_U]$. The bounds $a_U$, of the uniform distributions, have been optimized in \cite{fabiani2024random,fabiani2021numerical,fabiani2023parsimonious}.\par 
For the branch network, we have implemented two types of embeddings: 
\begin{itemize}
    \item  Linear random Johnson-Lindenstrauss (JL) embeddings \cite{johnson1984extensions}, in which case, we denote the branch feature vector $\bm{B}=(B_1,\dots,B_M)$ as:
\begin{equation}
    \bm{B}=\bm{\phi}_M^{br}(\bm{U})=\bm{\phi}^{JL}_M(\bm{U})=\frac{1}{\sqrt{M}}R\, \bm{U},
    \label{eq:JL_embedding}
\end{equation}
where $R$ is a matrix with elements that are sampled from a standard Gaussian distribution and $\bm{U}$ is the vector of function evaluation in the input grid.
    \item Nonlinear random embeddings \cite{rahimi2008uniform}. Here, for our illustrations, we use a random Fourier feature network (RFFN) \cite{rahimi2007random}, as embedding of the functional space:
\begin{equation}
    \bm{B}=\bm{\varphi}_M^{br}(\bm{U};\bm{\alpha}^{br})=\bm{\varphi}^{RFFN}_M(\bm{U};\bm{\alpha}^{br},\bm{b}^{br})=\frac{1}{m}\sqrt{\frac{2}{M}}[\cos(\bm{\alpha}^{br}_1\cdot \bm{U}+b^{br}_1),\dots,\cos(\bm{\alpha}^{br}_M\cdot \bm{U}+b^{br}_M)],
    \label{eq:RFFN_embedding}
\end{equation}
where we have two vectors of random variables $\bm{\alpha}^{br}$ and $\bm{b}^{br}$, 
from which we sample $M$ realizations. The weights $\bm{\alpha}^{br}_i$  are i.i.d. sampled from a standard Gaussian distribution, and the biases $\bm{b}^{br}_i$ are uniformly distributed in $\mathcal{U}[0,2\pi]$. This explicit random lifting has a low distortion for a Gaussian shift-invariant kernel distance.
\end{itemize} 
The training of RandONets reduces to the solution of a linear least-squares problem in the unknowns $W$, i.e., the external weights of the weighted inner product as in Eq.~\eqref{eq:branch_trunk}.
In what follows, before presenting the numerical implementation, we first present the treatment of aligned and unaligned data. The aligned data case refers to datasets where the training pairs are consistently organized, say in a grid, facilitating the learning process. In such scenarios, the network can more effectively learn the underlying nonlinear operator mappings due to the structured form of the data. On the other hand, training for unaligned data presents challenges compared to the aligned data case. In such scenarios, where the input and output pairs are not consistently organized, the network must learn to identify and map the complex relationships between disjoint datasets. This lack of alignment can make it more difficult for the network to capture the underlying nonlinear operator mappings accurately. In general, achieving high accuracy in this context often demands greater computational resources and more extensive hyperparameter tuning to ensure the network converges to an optimal solution.

\paragraph{RandONets for aligned data.}
Let us assume that the training dataset consists of $s$ sampled input functions at $m$ collocation/grid points. Thus, the input is included in a matrix $U \in \R^{m\times s}$. Let us also assume that the output function can be evaluated on a fixed grid of $n$ points $\bm{y}_k \in \R^d$, which are stored in a matrix  $Y\in \R^{n\times d}$ (row-vector); $d$ is the dimension of the domain. In this case, we assume that for each input function $u$, we have function evaluations $v$ at the grid $Y$ stored in matrix $V=\mathcal{F}[U] \in \R^{n\times s}$.

While this assumption may appear restrictive at a first glance (as for example some values in the matrix $V$ could be missing, or $Y$ can be nonuniform, and may change in time), nonetheless, for many problems in dynamical systems and numerical analysis, such as the numerical solution of PDEs, entails employing a fixed grid where the solution is sought. This is clearly the case in methods like Finite Difference or Finite Elements-based numerical schemes without mesh adaptation.
Additionally, even in cases where the grid is random or adaptive, there is still the opportunity to construct a ``regular" output matrix $V$ through ``routine" numerical interpolation of outputs on a fixed regular grid.
Now, given that the data are aligned, following Eq.~\eqref{eq:branch_trunk}, we can solve the following linear system (double-sided) of $n\times s$ algebraic equations in $N\times M$ unknowns:
\begin{equation}
    V=\mathcal{F}[U]= \bm{\varphi}_n^{tr}(Y;\bm{\alpha}^{tr}) \,  W  \, \bm{\varphi}_m^{br}(U;\bm{\alpha}^{br})=T(Y) \, W \, B(U).
    \label{eq:RandONets}
\end{equation}
Let us observe that --differently from a classical system of equations (e.g., $Ax=b$), here we have two matrices from the trunk and the branch features, that multiply the readout weights $W$ on both sides.\par
Although the number of unknowns and equations appears large due to the product, the convenient alignment of the data allows for effective operations that involve separate and independent (pseudo-) inversion of the trunk/branch matrices $T(Y)\in \R^{n\times N}$ and $B(U)\in\R^{M\times s}$.
Thus, the solution weights of Eq.~\eqref{eq:RandONets}, can be found by employing methods such as the Tikhonov regularization \cite{golub1999tikhonov}, truncated singular value decomposition (SVD), QR/LQ decomposition and Complete Orthogonal Decomposition (COD)\cite{hough1997complete} of the two matrices, as we will detail later, obtaining:
\begin{equation}
    W=\big(T(Y)\big)^{\dagger} \, V \, \big(B(U)\big)^{\dagger}.
    \label{eq:RandONets_solve}
\end{equation}
As one might expect, the trunk matrix typically features smaller dimensions compared to the branch matrix. This is because the branch matrix may involve numerous samples $s$ of functions (usually exceeding the number $n$ of points in the output grid), along with a higher number of neurons $M$ required to represent the high-dimensional function input, as opposed to the 
$N$ neurons of the RP-FNN trunk which embeds the input space.
At the end, the computational cost associated with the training (i.e., the solution of the linear least-squares problem) of RandONets is of the order $O(M^2 s+s^2M)$. 
Here, we use the COD algorithm \cite{hough1997complete} for the inversion of both $T$ and $B$ matrices (for a comparison of truncated SVD and COD algorithms for RP-FNN training see in \cite{fabiani2024random}).

\paragraph{RandONets for unaligned data.}
In contrast to the aligned data, 
the output $V$ cannot be usually stored in a matrix, but we have to consider a (long) vector.
To address learning with unaligned data, it is sufficient to assume that for each input function $u$, the output $v(\bm{y})$ is available at a \emph{single} random location in the output domain. This encompasses scenarios with a sparse random grid, where each output in the grid is treated separately, yet necessitating the introduction of multiple copies of the function $u$. Thus, let us assume we have stored the input functions in a matrix $U\in \R^{m\times S}$ and a vector of outputs $V \in \R^{1\times S}$. 

Here, $S\in\N$ denotes both the total number of output functions and the total number of input functions. Unlike the aligned case, we store the random points for each input in the matrix $Y\in \R^{d\times S}$, where $d$ now represents the columns instead of the rows, and $S$ reflect the total number of (single) random locations where the individual outputs are sought.\par
Now, returning to Eq.~\eqref{eq:branch_trunk}, we notice that with the current format of inputs (both column-wise), we can express the output using the Hadamard (Shur) product ($\otimes$):
\begin{equation}
    V=\sum_{k=1}^N\sum_{i=1}^M w_{ki} T_k(Y) B_i(U)=\sum_{k=1}^n T_k(Y) \otimes \bm{w}_k \, B(U),
\end{equation}
where $\bm{w}_k$ are the rows of the matrix $W$. This corresponds to the original formulation of the DeepONet by Lu et al. (2021) \cite{lu2021learning}, where instead of considering the merging of the branch and trunk networks as a weighted inner product, the focus is on the individual output at a single location, rather than treating the output as an entire transformed function.
In this scenario, to solve the linear least-squares problem in terms of the $N\times M$ unknown weights $W$, we need to reshape the matrix $W$ into a row vector $\bm{\omega}$, where the elements $\omega_q=w_{ki}$, with $q=k+(i-1)n$. Then we construct the full collocation matrix $Z \in \R^{NM \times S}$, such that the rows $\bm{z}_q$ are obtained as:
\begin{equation}
    \bm{z}_q=T_k\otimes B_i, \qquad q=k+(i-1)N.
\end{equation}
Note that the Hadamard product $T_k\otimes B_i$ is possible as both lie in $\R^S$.\par
At this point, the weights $\bm{\omega}$ can be computed, through a (pseudo-) inversion of the matrix $Z$, in analogy to what was detailed in section \ref{sec:RPNN}, resulting in:
\begin{equation}
    \bm{\omega}=Y\, Z^{\dagger}.
\end{equation}
We note that the total number $S$ of single outputs can be viewed as proportional to the product of $s$ (the number of different input functions) and $N$ (the number of points in the output grid), as explained in the aligned case. 
In particular, employing an unaligned training algorithm for aligned data (by augmenting the input function with copies and reshaping the output) will result exactly in $S=Ns$. Now, the pseudo-inversion of the matrix $Z$ will result in a computational cost of the order $O\big((Ns)^2MN+(MN)^2Ns\big)$, which is significantly higher.  For instance, in the case the values of $M$, $s$ are similar/proportional to $N$, 
we obtain a transition, in terms of computational complexity, from an order $O(N^3)$ for the aligned case to an order $O(N^6)$ for the unaligned case.\par
To this end, we argue that the unaligned approach described here should only be considered if the output data display substantial sparsity, suggesting that the random output grid does not adequately represent the output function. Conversely, we advocate for prioritizing the aligned approach in other scenarios. Even if it entails performing interpolation on a fixed grid to generate new aligned output data.

\paragraph{Numerical implementation of the training of RandONets.}
\label{sec:RPNN_solve}
Below, we provide more details on the training process of RandONets. 
From a numerical point of view, the resulting random trunk $T(Y)$ and branch $B(U)$ matrices, tend to be ill-conditioned. Therefore, in practice, we suggest solving Eq.~\eqref{eq:RandONets} as described in  \ref{sec:RPNN} via a truncated SVD (tSVD), Tikhonov regularization, QR decomposition or regularized Complete Orthogonal decomposition (COD) \cite{hough1997complete,fabiani2024random}. Here, we describe the procedure for the branch network. For the trunk matrix, the procedure is similar.\par 

The regularized pseudo-inverse $(B(U))^{\dagger}$, for the solution of the problem in Eq.~\eqref{eq:RandONets_solve} is computed as:
\begin{equation}
    B(U)=U\Sigma V^T=[U_r \quad \tilde{U}]
    \begin{bmatrix}
        \Sigma_k & 0\\
        0 & \tilde{\Sigma}
    \end{bmatrix}
     [V_k \quad \tilde{V}]^T
    , \qquad (B(U))^{\dagger}=V_k\Sigma_k^{-1}U_r^T,
    \label{eq:pseudo_inverse}
\end{equation}
where the matrices $U=[U_k \quad \tilde{U}]\in \R^{k\times n}\oplus\R^{(n-k)\times n}$ and $V=[V_k \quad \tilde{V}] \in \R^{k\times s}\oplus\R^{(s-k)\times s}$ are orthogonal and $\Sigma \in \R^{n\times s}$ is a diagonal matrix containing the singular values $\sigma_i=\Sigma_{(i,i)}$. Here, we select the $k$ largest singular values exceeding a specified tolerance $0<\epsilon\ll 1$, i.e., $\sigma_1,\dots,\sigma_k>\epsilon$, effectively filtering out insignificant contributions and improving numerical stability.\par  
Alternatively, a more robust method that further enhances numerical stability, involves utilizing a rank-revealing LQ decomposition (transposition of the QR decomposition, which can be used for the trunk matrix $T(Y)$) with column-pivoting:
\begin{equation}
    P\, B(U)= [L \quad 0]\begin{bmatrix}
        Q_1\\
        Q_2
    \end{bmatrix}, 
\end{equation}
where (e.g., if $n>s$) the matrix $Q=[Q_1 \quad Q_2] \in \R^{s\times n}\oplus\R^{(n-s)\times n}$ is orthogonal, $L \in \R^{s \times s}$ is a lower triangular square matrix and the matrix $P \in \R^{n \times n}$ is an orthogonal permutation of the columns. The key advantage of the column permutations lies in its ability to automatically identify and discard small values that contribute to instability. Indeed, in case of ill-conditioned matrices, we have that effectively the rows of the matrix $Q$ do not span the same space as the rows of the matrix $B(U)$. As a result, the matrix $L$ is not full lower triangular, but we have:
\begin{equation}
    B(U)=P^T\begin{bmatrix}
        L_{11} & 0\\
        L_{21} & 0
    \end{bmatrix}
    \begin{bmatrix}
        Q_1\\
        Q_2
    \end{bmatrix},
    \label{eq:QR1}
\end{equation}
where, if $rank(B(U))=r<n$, the matrix $L_{11} \in \R^{r \times r}$ is effectively lower triangular and $L_{21} \in \R^{(s-r)\times r}$ are the remaining rows. Note that numerically, one selects a tolerance $0<\epsilon<<1$ to estimate the rank $r$ of the matrix $B(U)$ and set values of $B(U)$ below the threshold to zero. 
Then, to find the pseudo inverse of the branch matrix, we can additionally use the Complete Orthogonal Decomposition (COD) \cite{hough1997complete}, by also computing the LQ decomposition of the transposed non-zero elements in $L$ (for the trunk matrix this will correspond to a second QR decomposition):
\begin{equation}
    [L_{11}^T\quad L_{21}^T]=[T_{11}^T\quad 0] V.
\end{equation}
Finally, by setting $S^T=VP$, one obtains:
\begin{equation}
    B(U)=S \begin{bmatrix}
        T_{11} & 0\\
        0 & 0
    \end{bmatrix}
    \begin{bmatrix}
        Q_1\\
        Q_2
    \end{bmatrix},
    \label{eq:COD}
\end{equation}
where $T_{11}$ is an upper triangular matrix of size $r \times r$. 
Note that the inversion of the matrix $T_{11}$, can be efficiently computed using the back substitution algorithm that is numerically stable.

\section{Numerical Results}
\label{sec:numerical}
In this section, we present numerical results focusing exclusively on the aligned case within the framework of RandONets. This choice stems from (a) the observed superior computational cost of the proposed aligned approach; and (b) the aligned case fits well with the nature of the considered PDEs problems in dynamical systems and numerical analysis approaches. For a first proof of concept, our investigation encompasses a selection of benchmark problems originally addressed in the paper introducing DeepONet by Lu et al. (2021) \cite{lu2021learning}. Additionally, we focus on other benchmark problems concerning the evolution operator (right-hand-side (RHS)) of PDEs. Through these numerical experiments, we aim to demonstrate the effectiveness and versatility of RandONets in tackling a diverse array of challenging problems in dynamical systems and scientific computing for the solution of the inverse problem.\par 
In \cite{lu2021learning}, there is a discussion on how to generate the dataset of functions: they compare Gaussian Random Fields and other random parametrized orthogonal polynomial sets. Here, we decided to generate the input-output data functions without using precomputed datasets. We consider a random parametrized RP-FNN with $200$ neurons and Gaussian radial basis functions combined with few additional random polynomial terms.
\begin{equation}
    u(t)=\bm{w} \exp\big( \bm{s} (t-\bm{c})^2 \big)+ a_0+a_1 \,t+a_2\,t^2,
    \label{eq:basisRPNN}
\end{equation}
where the parameters $\bm{w},\bm{s},\bm{c} \in \R^{200}$, 
$a_0, a_1, a_2 \in \R$ are all randomized to generate the dataset of input functions.\par 
In all numerical examples, we select many different realizations of these functions. In some of these, in selecting the training datasets, we distinguish the case in which we utilize \emph{limited-data} for training. In particular, we utilize $15\%$ of the data for the training set and $85\%$ for the test set. In the case in which we assume that we have available \emph{extensive-data}, we utilize $80\%$ of the data for the training set and $20\%$ for the test set.\par
The range of the values of the parameters ($\bm{w},\bm{s},\bm{c},a_0, a_1, a_2$) is detailed for each case study. When possible, the output function is computed analytically. Otherwise, a well-established numerical method (with sufficiently small tolerances) is used to compute accurate solutions as the ``ground truth".
To represent both the input functions $u$ and output functions $v$, we use an equally spaced grid of 100 points in the domains $\mathsf{K}_1$ and $\mathsf{K}_2$ of interest. In particular, in all examples considered here, we take as input, one-dimensional domains (intervals in $[a,b]$).\par

Regarding the architecture of the RandOnets, as also detailed in Section \ref{sec:RandONets_implementation}, we investigate and compare the performance of two different RandOnet architectures, with two well-established embedding techniques, respectively: linear random Johnson-Lindenstrauss (JL) embeddings denoted by $\phi_{M}^{JL}$ (as presented in Eq. \eqref{eq:JL_embedding}) and Random Fourier Feature Network (RFFN) embeddings denoted by  $\phi_{M}^{RFFN}$ (as presented in Eq. \eqref{eq:RFFN_embedding}). These architectures will be subsequently referred to as RandOnets-JL and RandOnets-RFFN, respectively. We will explore the impact of varying the number of neurons ($M$) within the single hidden layer of the branch, effectively controlling the dimension of the branch embedding.
For both RandOnets-JL and RandOnets-RFFN, the trunk embedding leverages a non-linear RP-FNN architecture denoted by $\phi_N^{tr}$ (as presented in \eqref{eq:trunk_embedding}), which utilizes hyperbolic tangent activation functions $\psi$ and parsimoniously function-agnostic randomization of the internal weights (as described in \cite{fabiani2024random,fabiani2023parsimonious,fabiani2021numerical}). Throughout the experiments, we will maintain a consistent number of neurons ($N=200$) within the trunk's hidden layer, thus ensuring a fixed size for the trunk embedding. It is important to note that the RandOnet-JL architecture incorporates a combination of linear and non-linear embedding techniques, whereas the RandOnet-RFFN architecture is entirely non-linear.\par

Given the big difference in the computational cost for the inversion of the branch matrix compared to the trunk matrix, we decided to fix the number of neurons $N=200$ in the trunk RP-FNN embedding for both the RandOnets-JL and RandOnets-RFFN. The inversion of the corresponding trunk matrix $T\in\R^{100\times 200}$m thus it is, for any practical purposes, relatively negligible. Here, we investigate the performance of the scheme for $M=10, 20, 40, 80, 100, 150, 300, 500, 1000, 2000$ neurons in the branch embedding of both RandOnets-JL and RandOnets-RFFN and the corresponding increment in computational cost.
\paragraph{Metrics.}
To assess the performance of the RandONets, we utilize both the mean squared error (MSE) for the entire test set, as well the $L^2$--error for each output-function in the test set. In particular, we report the median $L^2$ and the percentiles $5\%-95\%$.
Importantly, we report the execution time in seconds of the scheme, thus indicating when the computations are performed with GPU or CPU.
\paragraph{Remark on the DeepOnet architectures used.}
Given the high-computational cost associated when training DeepOnets with the Adaptive Moment Estimation (Adam) algorithm, (even if we employ a GPU hardware), we do not focus now on performing a convergence diagram of the scheme or finding the best architecture. For our illustrations, we just selected a few configurations. In particular, we selected $2$ hidden layers for both trunk and branch networks, each layer with a prescribed number $N=\{5,10,20,40\}$ of neurons. Also, we employ hyperbolic tangent as activation functions for both branch and trunk networks. We will refer to the performance of these vanilla DeepONets in Tables with the notation $[N,N]$. We remark that the number of free trainable parameters $\zeta$ of such DeepOnet configurations is $m\times N+3N\times N+5N$, (e.g., $N=40, m=100$, then $\zeta=9000$) which is not higher than the biggest considered RandONet. Indeed, RandOnets have a number of free trainable parameters equal to $N\times M$ (e.g., in the biggest case considered here, $N=200$, $M=2000$, it corresponds to $\zeta=400,000$ parameters). However, as we will show, despite the high number of parameters, such RandONets can be trained in around one second. 
Finally, we also remark that when employing a gradient-descent based algorithm, as the Adam one, there is no guarantee of convergence, and the generalization of the network can be moderate. We anyway decided to train DeepOnet for a fixed number of iterations equal to $20,000$ for ODE benchmarks and to $50,000$ for the PDEs.
\paragraph{Remark on the hardware and software used.}
In our experiments, we utilized the DeepONet framework implemented in Python with the DeepXde library \cite{lu2021deepxde}, leveraging TensorFlow as the back-end for computations. These computations are executed on a GPU NVIDIA GeForce RTX 2060, harnessing its parallel processing capabilities to expedite training and evaluation. Additionally, we ran the Python code on \texttt{Google Colab} using a Tesla K80 GPU, resulting in computational times approximately 6 to 7 times slower than those reported in the main text. While we do not include these execution times in the main text, we mention them here as a reference. In contrast, the RandONets framework is implemented in MATLAB 2024a and executed on a single CPU of an Intel(R) Core(TM) i7-10750H CPU @ 2.60GHz, with 16 GB of RAM.\par
It is worth noting that while the hardware and software environments for the two frameworks differ significantly, hindering a direct comparison of computational times, we report the computational times for both approaches for transparency. Despite the disparity in hardware and software, it is noteworthy that the single CPU utilized is not suitable for the Python code, whereas it proves to be efficient for the RandONets implemented in MATLAB. Additionally, while we acknowledge that differences in computational times between MATLAB and Python platforms may arise due to a range of factors beyond software differences alone, we believe that the reported differences cannot be just explained by the two different software implementations. We provide these computational costs as they are observed, recognizing that, while a hand-made MATLAB implementation of DeepONet is feasible, our objective is to compare with the professionally implemented and widely used DeepXde library \cite{lu2021deepxde} to ensure a reliable comparison.\par

\subsection{Some simple (pedagogical) ODE benchmark problems}
\label{sec:benchmarkODE}
We start with some very simple benchmark problems involving non-autonomous ODEs subject to a time dependent source term $u(t)$, in the form:
\begin{equation}
    \frac{dv(t)}{dt}=f(t,v)+u(t), \qquad v(0)=v_0, \quad t\in[0,T],
\end{equation} with some forcing time-dependent input function $u(t)$. The solution function $v(t)$ depends directly on the forcing term, the function $u(t)$. Thus, there exists an operator that maps $u(t)$ into $v(t)$.
The task here, different from the one for the PDEs, is to learn the ``solution operator'' for one initial condition. Of course, learning the full solution operator would need a set of different initial conditions as in \cite{lu2021learning} but as mentioned these serve purely for pedagogical purposes. Still, the focus, in this first work, is on the approximation of the evolution operators of PDEs; we will present these results in the subsection \ref{sec:benchmarkPDE}.
\begin{figure}[ht!]
    \centering
    \subfigure[DeepOnet (extensive-data)]{
        \includegraphics[width=0.31\textwidth]{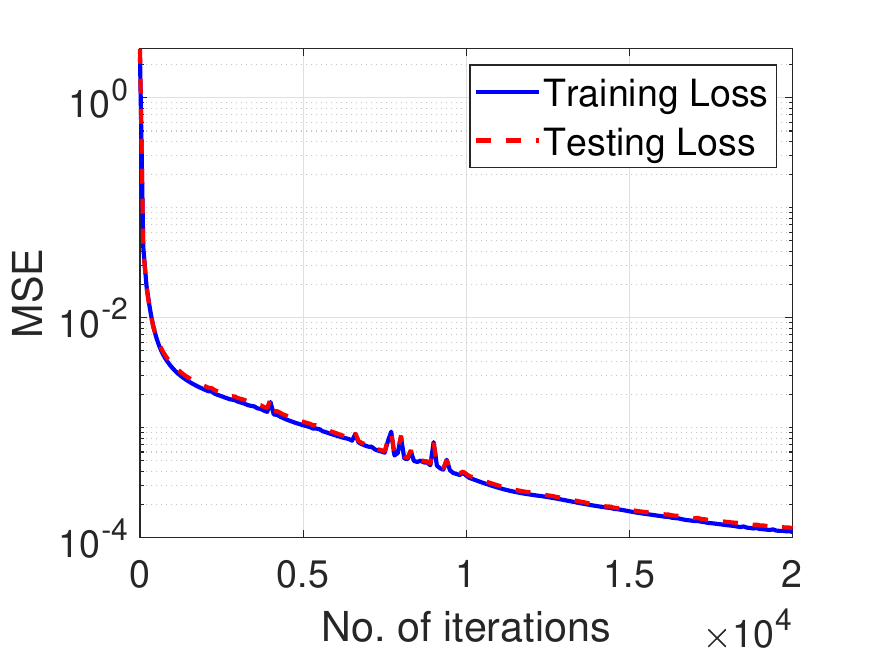}
    }
    \subfigure[RandONets (extensive-data)]{
    \includegraphics[width=0.31\textwidth]{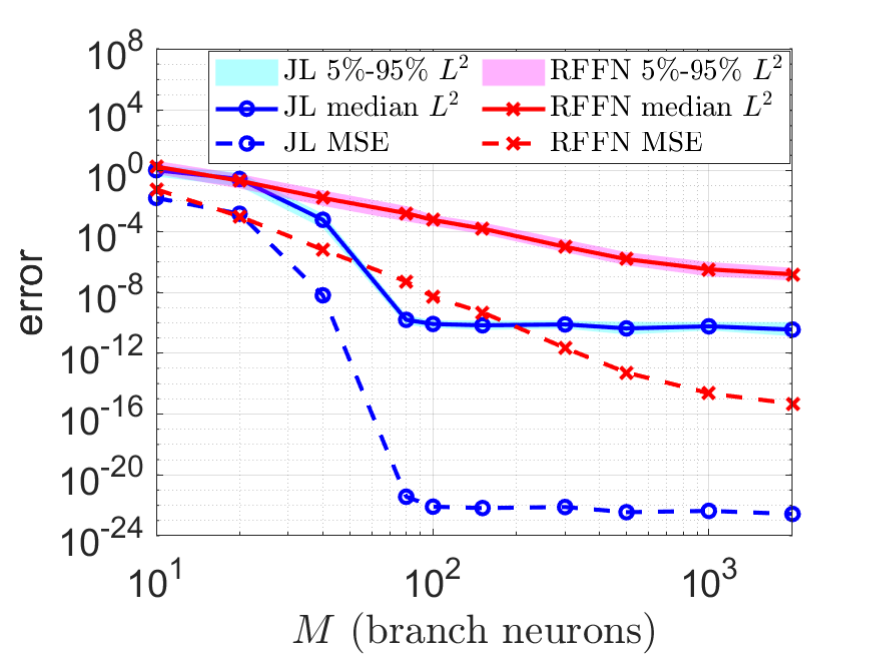}
    }
    \subfigure[RandONets (extensive-data)]{
    \includegraphics[width=0.31\textwidth]{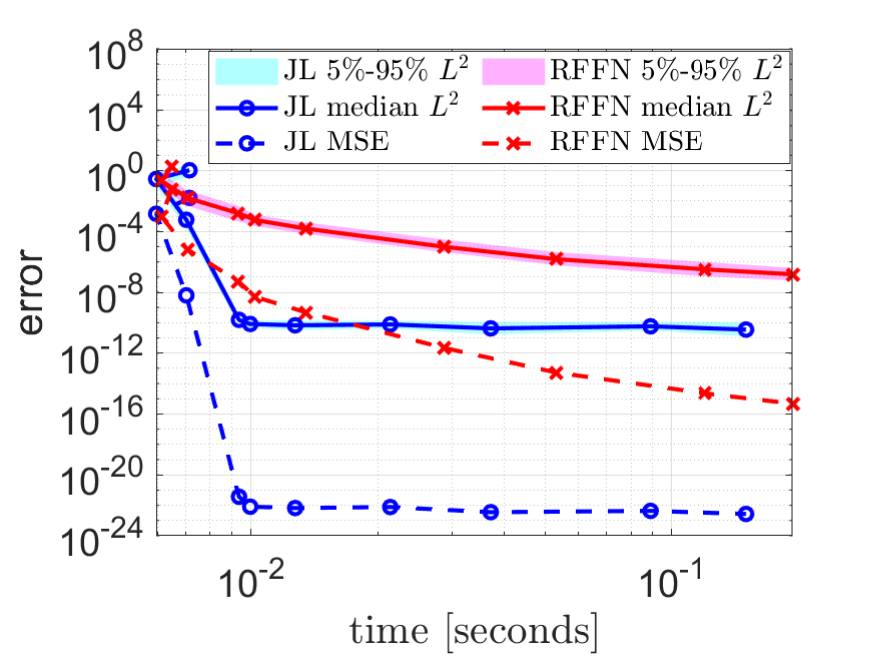}
    }
    \subfigure[DeepOnet (limited-data)]{
    \includegraphics[width=0.31\textwidth]{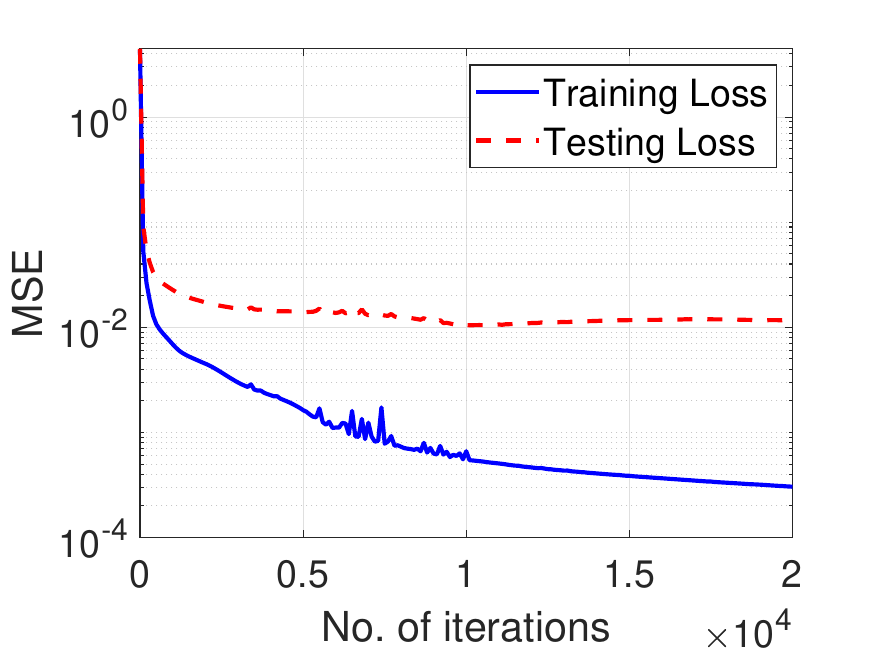}
    }
    \subfigure[RandONets (limited-data)]{
    \includegraphics[width=0.31\textwidth]{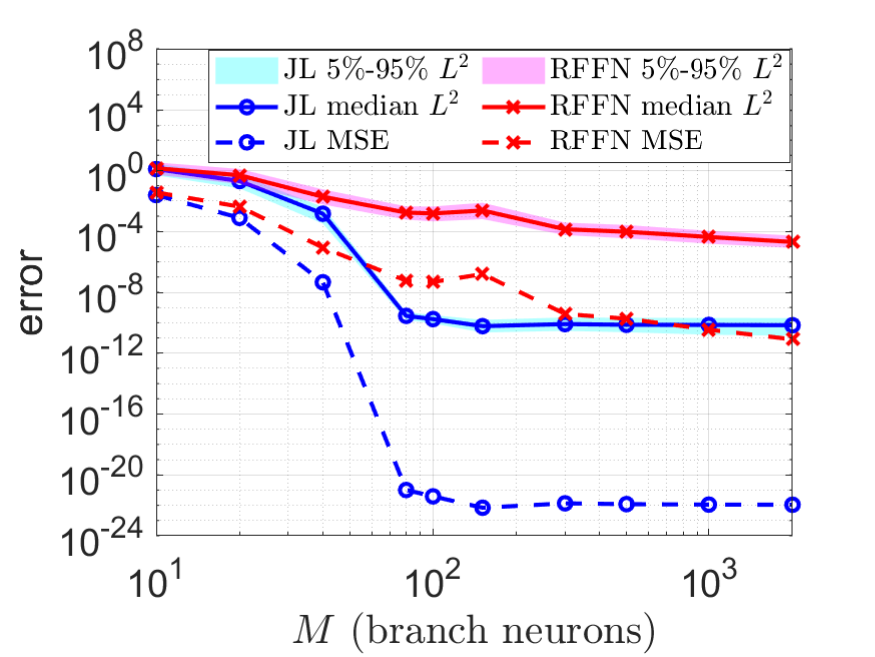}
    }
    \subfigure[RandONets (limited-data)]{
    \includegraphics[width=0.31\textwidth]{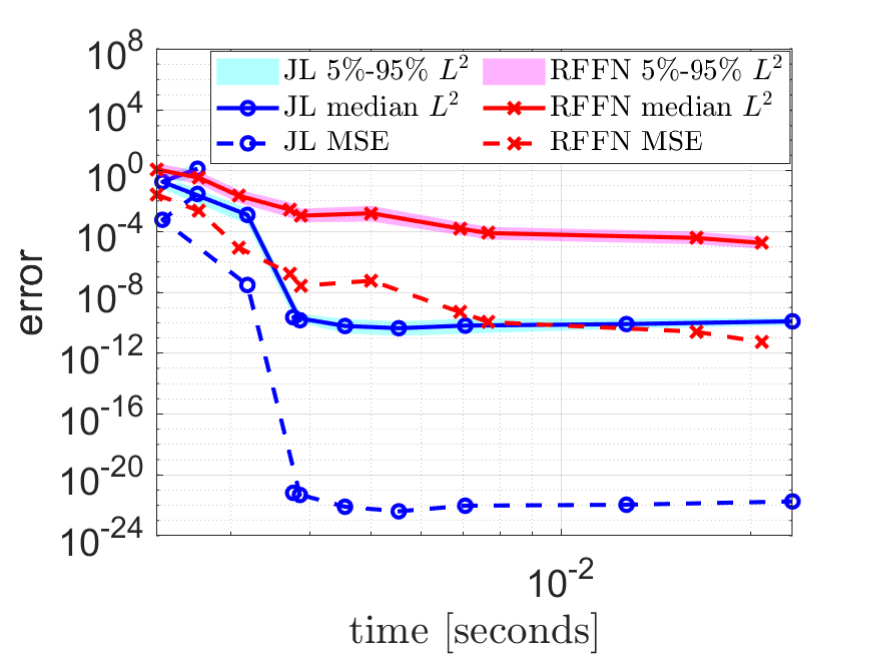}
    }
    \caption{Case study 1: Antiderivative Operator, in Eq. \eqref{eq:antiderivative}. (First row) extensive-data case, $800$ training input functions; (Second row) limited-data case, $150$ training input functions. (a), (d) MSE for the training and test sets, with the DeepOnet with $2$ hidden layers (indicatively) with $40$ neurons each, for both the branch and trunk networks. (b), (c), (e), (f) MSE and $L^2$ error, $5\%-95\%$ range and median, of the RandONets, for different size $M$ of the branch embedding. The errors are computed w.r.t only the output functions in the test dataset. Comparison of Johnson-Lindenstrauss (JL) branch embedding with random Fourier features (RFFN) embeddings. We set the size of the Trunk network to $N=200$ and the grid of input points to $m=100$. Numerical approximation accuracy vs. (b)-(e) the number of neurons $M$ in the hidden layer of the branch network; and (c)-(f) vs. computational times in seconds.}
    \label{fig:antiderivative_conv}
\end{figure}

\subsubsection{Case study 1: Anti-derivative operator}
The first benchmark problem that we consider is the anti-derivative operator\cite{lu2022learning}, thus the solution $v(t)$ given the function $u(t)$ of the following (phase-independent, $f(t,v)\equiv 0$) ODE problem:
\begin{equation}
    \frac{dv(t)}{dt}=u(t), \qquad v(0)=v_0, \quad t\in[0,1],
    \label{eq:antiderivative}
\end{equation}
thus learning the linear operator $\mathcal{F}[u](t)=v(t)$. The corresponding analytical anti-derivatives are:
\begin{equation}
    v(t)=\bm{w}\biggl(-\sqrt{\pi}\text{erfi}\big( \sqrt{\bm{s}} (\bm{c}-t) \big)/(2\sqrt{\bm{s}})\biggr) + a_0 t+a_1 t^2/2+a_2 t^3/3+C,
\end{equation}
where $\text{erfi}$ is the error function and $C$ is a constant that has to be fit by the initial condition $v(0)=v_0$. We select $v_0=0$ and we set $\tilde{v}(t)=v(t)-v(0)$ as the output function.

The values of the parameters $\bm{w},\, a_0,\, a_1,\, a_2 \sim \mathcal{U}[-1,1]$, $\bm{s} \sim \mathcal{U}[0,500]$ and $\bm{c} \sim \mathcal{U}[0,1]$, of the RP-FNN based function dataset, as in Eq.~\eqref{eq:basisRPNN}, are (element-wise) sampled from the corresponding aforementioned uniform distributions.
To generate the data, we used $1000$ random realizations. We considered two different sizes of training sets. In particular, we used $15$\% for the training and $85\%$ for the test set (we remind the reader that we call this limited-data case). As described above, for the extensive data-case we used $80\%$ for the training and $20\%$ for the test set.

\begin{table}[ht!]
\centering
\caption{Case study 1: Anti-derivative Operator in Eq. \eqref{eq:antiderivative}.  We report Mean Squared Error (MSE), percentiles (median, $5\%$, $95\%$ of $L^2$ error across the test set. The extensive-data case comprises $800$ training functions, while the limited-data case uses $150$ functions as training. We employed vanilla DeepOnets with $2$ hidden layers, denoted as $[N,N]$ neurons, for both the branch and the trunk. We set $N=5,10,20,40$. DeepOnets are trained with $20,000$ Adam iterations (with learning rate $0.001$ and then $0.0001$). We report the RandONet encompassing Johnson-Lindenstrauss (JL) Featured branch network (with $M=100$ neurons) as well as the Random Fourier Feature branch Network (RFFN) (with $M=2000$ neurons).}
\begin{tabular}{|c|c|c|c|c|c|c|}
\hline
                      data &    ML-model       & MSE  & 5\% $L^2$ & median--$L^2$ & 95\% $L^2$ & comp. time\\
                      \hline
\multirow{2}{*}{80\%} & DeepOnet $[5,5]$  & 9.83E$-$01  &  2.31E$+$00  &  6.90E$+$00  &   1.80E$+$01   &   4.75E$+$02   (GPU)     \\
& DeepOnet $[10,10]$  & 2.28E$-$03  &  2.43E$-$01  &  4.37E$-$01   &  7.46E$-$01  &    4.62E$+$02   (GPU)     \\
& DeepOnet $[20,20]$  & 4.39E$-$04  &  1.26E$-$01  &  2.00E$-$01  &   2.97E$-$01  &    4.79E$+$02   (GPU)     \\
& DeepOnet $[40,40]$  & 1.22E$-$04 &   7.04E$-$02  &  1.03E$-$01  & 1.60E$-$01  & 5.23E$+$02   (GPU)     \\
                      & RandONet--JL ($100$) & 9.43E$-$23 & 4.33E$-$11  & 8.01E$-$11  & 1.68E$-$10 & 1.02E$-$02 (CPU) \\
                      & RandONet--RFFN ($2000$) & 8.09E$-$16 & 6.81E$-$08 & 1.73E$-$07  & 5.99E$-$07  & 1.96E$-$01 (CPU) \\
                      \hline
\multirow{2}{*}{15\%}  & DeepOnet $[5,5]$  & 8.88E$-$02  &  1.08E$+$00   & 2.48E$+$00  &   5.15E$+$00   &   1.05E$+$02   (GPU)     \\
& DeepOnet $[10,10]$ & 2.99E$-$03   & 2.73E$-$01 &   4.91E$-$01  &   8.51E$-$01  &    9.78E$+$01   (GPU)     \\
& DeepOnet $[20,20]$  & 7.48E$-$04  &  1.51E$-$01 &   2.54E$-$01  &   4.07E$-$01   &   1.13E$+$02    (GPU)     \\
& DeepOnet $[40,40]$  & 1.16E$-$02 & 2.57E$-$01  & 7.36E$-$01 & 2.12E$+$00  & 1.24E$+$02  (GPU)     \\ 
                      & RandONet--JL ($100$) & 1.66E$-$21 & 2.22E$-$10 & 3.74E$-$10  & 6.11E$-$10  & 3.60E$-$03  (CPU)       \\
                      & RandONet--RFFN (2000) & 8.12E$-$12 & 8.14E$-$06  & 2.03E$-$05 & 5.25E$-$05  & 1.88E$-$02 (CPU)       \\
                      \hline
\end{tabular}
\label{tab:antiderivative}
\end{table}
In Figure \ref{fig:antiderivative_conv}, we depict the numerical approximation accuracy for the test set in terms of the MSE and percentiles median, $5\%-95\%$ of $L^2$--errors. As shown, the training of the all RandONets takes approximately less than 1 second and is performed without iterations.
In Table \ref{tab:antiderivative}, we summarize the comparison results with the vanilla DeepONet in terms of the best accuracy and best computational times. For the RandONets, we used $100$ neurons for the JL embedding, and $2000$ neurons for the RFFN embedding for the branch network. As shown, the JL-based RandONets gets an astonishing almost machine-precision accuracy of $MSE\simeq1E$-$23$, $L^2\simeq1E$-$11$ with just $40$ neurons in the branch with a computational time of $\simeq 0.01$ seconds.
Such ``perfect'' results are due to the simplicity of the problem and its linearity. The nonlinear RFFN embedding result in a lower performance with respect to the JL RandOnets, for this linear problem, obtaining an $MSE\simeq 1E$-$16$ and a median $L^2\simeq 1E$-$08$, using $2000$ neurons in the branch, with a computational time of less of the order $0.1$ seconds. We employ vanilla DeepOnets with two hidden layers, denoted as $[N,N]$ neurons, in both Trunk and Branch. We observe a rather slow convergence in accuracy by increasing the size of the Vanilla DeepOnets. However, the vanilla DeepOnets need many iterations to reach an adequate accuracy. After $20,000$ iterations, the accuracy in the extensive-data case, $for N=40$, is around $1E$-$04$ in terms of MSE, but the $L^2$ error is on the order of $1E$-$01$. In the limited-data case, the vanilla DeepOnet, with $N=40$, gives a rather poor performance: the MSE on test data is stuck at $1E$-$02$ thus, overfitting. The corresponding $L^2$ error is on the order of $1$. Indeed, the DeepOnet with $N=20$ performs better. We can explain such failure due to difficult dataset considered, that needs sufficient input-ouput functions to be well represented.\par 

Our results for the two hidden layers vanilla DeepOnet are in line with the ones presented in \cite{lu2021learning}. Also, for investigations on different architectures one can refer to the same paper. In particular, there, for the vanilla DeepOnet with $4$ hidden layers and $[2560,2560,2560,2560]$ neurons in both trunk and branch, they report an MSE of around $1E$-$08$ after $50,000$ iterations.\par
As a matter of fact, for this case study, the execution times (training times) for RandONet, utilizing both linear JL and nonlinear RP-FNN random embeddings,  are 10,000 times faster, while achieving $L^2$ accuracy that is $6$ to $10$ orders of magnitude higher.

\subsubsection{Case study 2: Pendulum with external force}
We consider the motion of a simple pendulum, consisting of a point mass suspended from a support by a mass-less string of length $l=1$, on which act the gravity force and an additional external force $u(t)$.
\begin{figure}[ht!]
    \centering
    \subfigure[DeepOnet (extensive-data)]{
        \includegraphics[width=0.31\textwidth]{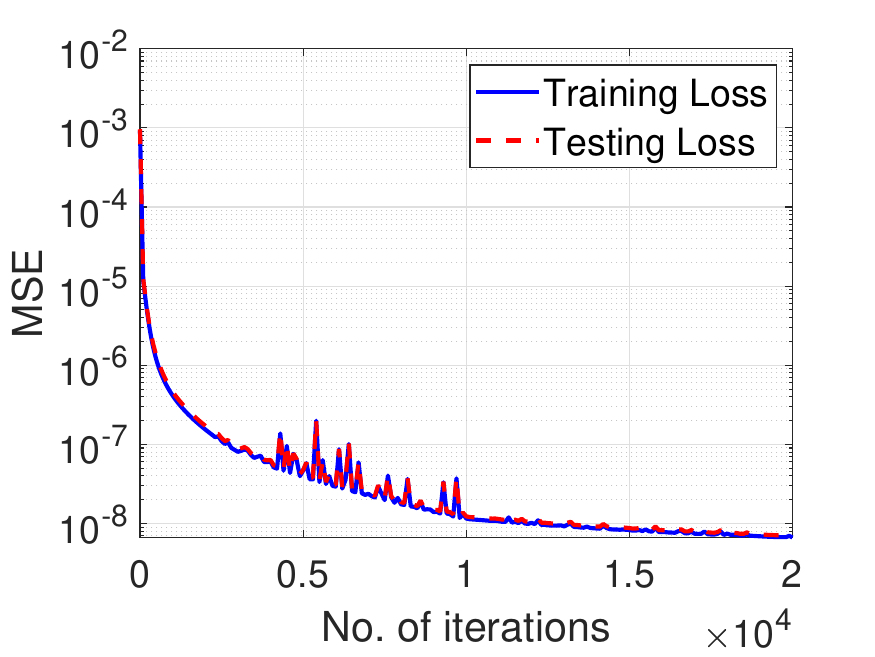}
    }
    \subfigure[RandONets (extensive-data)]{
    \includegraphics[width=0.31\textwidth]{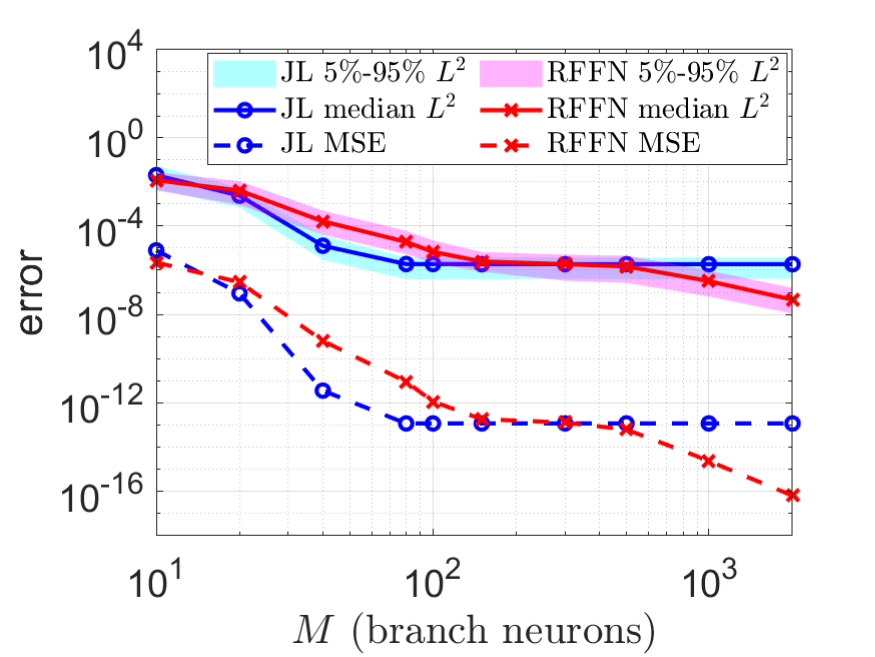}
    }
    \subfigure[RandONets (extensive-data)]{
    \includegraphics[width=0.31\textwidth]{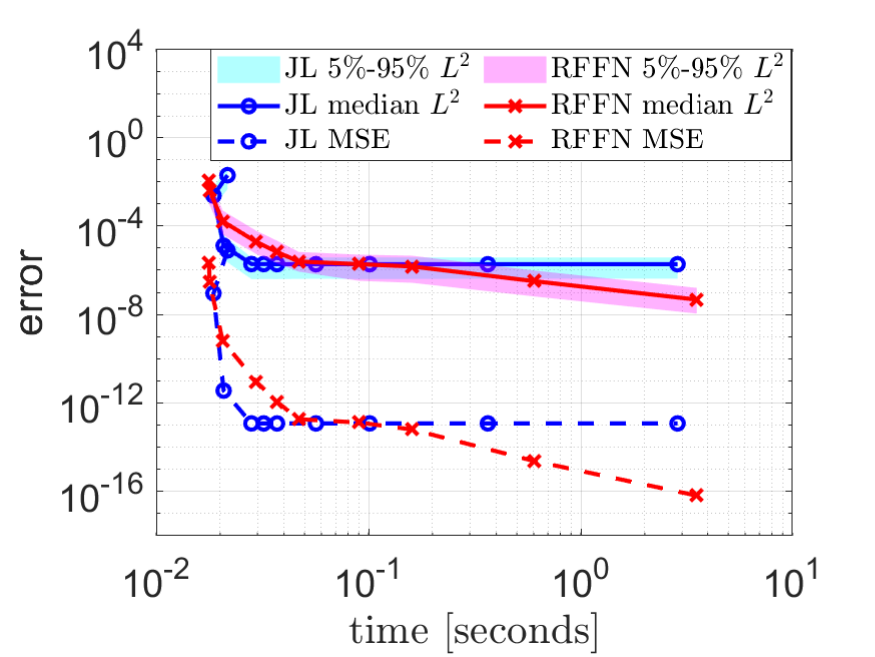}
    }
    \subfigure[DeepOnet (limited-data)]{
    \includegraphics[width=0.31\textwidth]{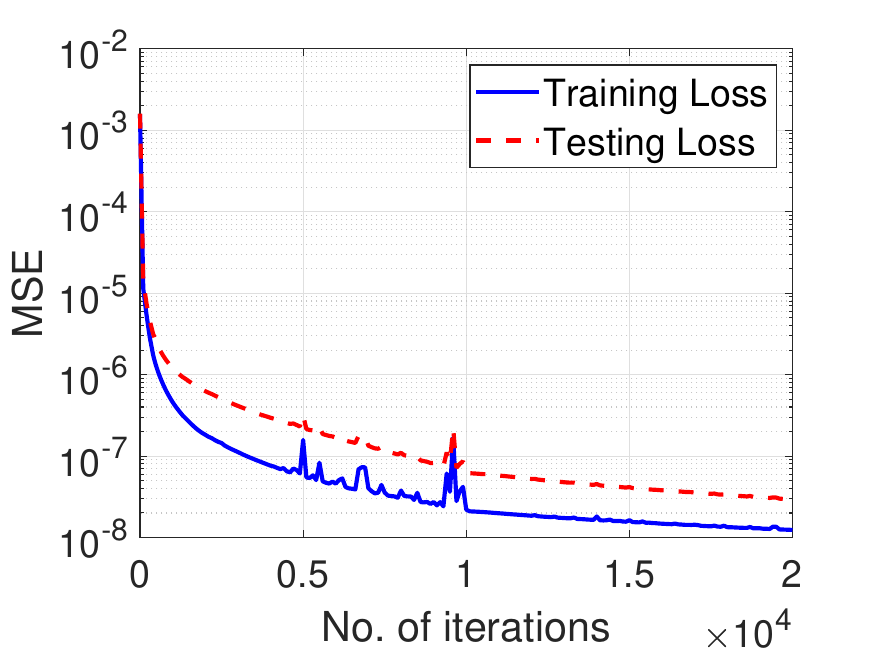}
    }
    \subfigure[RandONets (limited-data)]{
    \includegraphics[width=0.31\textwidth]{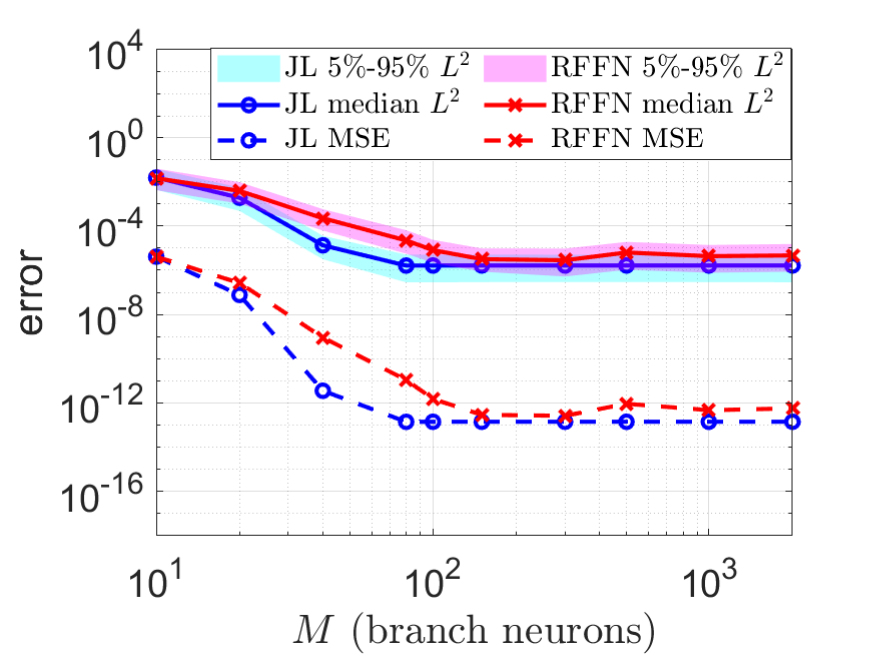}
    }
    \subfigure[RandONets (limited-data)]{
    \includegraphics[width=0.31\textwidth]{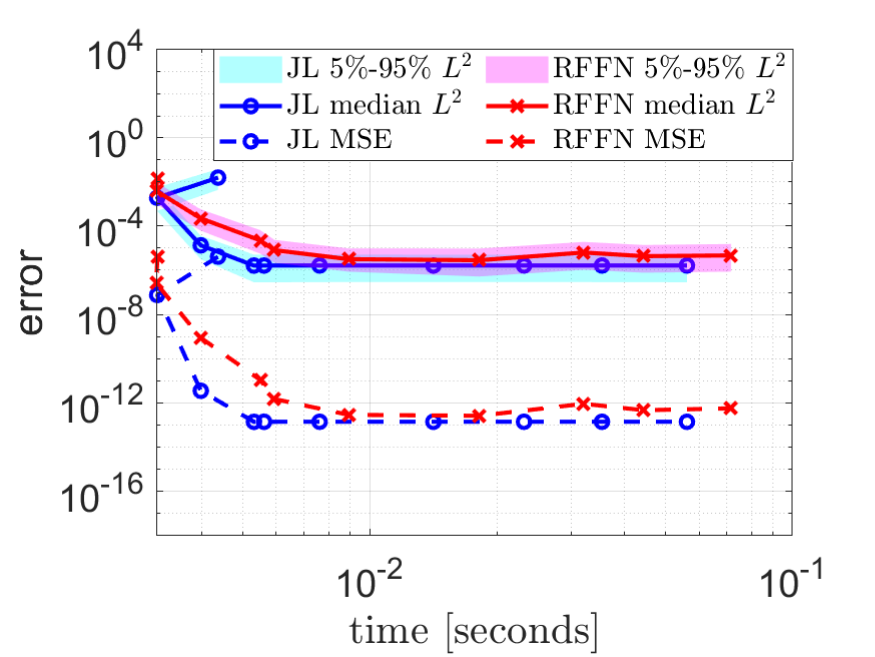}
    }
    \caption{Case study 2: Pendulum with external force, in Eq. \eqref{eq:pendulum}. (First row) extensive-data case, $2400$ training input functions; (Second row) limited-data case, $450$ training input functions. (a), (d) Convergence of training and test set MSE of the DeepOnet with two hidden layers (indicatively) with $40$ neurons each, for both the branch and trunk networks. (b), (c), (e), (f) MSE and $L^2$ error percentiles (median, $5\%-95\%$), of the RandONets, for different size $M$ of the branch embedding. The errors are computed w.r.t. only the output functions in the test dataset. Comparison of Johnson-Lindenstrauss (JL) branch embedding, in Eq. \eqref{eq:JL_embedding}, with Random Fourier Feature Networks (RFFN) embeddings, in Eq. \eqref{eq:RFFN_embedding}. We set the size of the Trunk network to $N=200$ and the grid of input points to $m=100$. Numerical approximation accuracy vs. (b)-(e) number of neurons $M$ in the hidden layer of the branch network; and (c)-(f) vs. computational time in seconds.}
    \label{fig:pendulum_conv}
\end{figure}
The system is described by the following second order ODE:
\begin{equation}
    \frac{d^2 v(t)}{dt^2}=-k\sin v(t)+u(t), \qquad t\in[0,1], \quad v(0)=0, \quad v'(0)=0,
    \label{eq:pendulum}
\end{equation}
where $v$ represents the angle with the vertical, and $k=9.81$. 

Again, the goal is to learn the operator that maps the input function $u(t)$ to the output function $v(t)$.
\begin{table}[ht!]
\centering
\caption{Case study 2: Simple pendulum with external force, as in Eq. \eqref{eq:pendulum}. We report Mean Squared Error (MSE) and percentiles (median, $5\%-95\%$) of $L^2$ error across the test set. The extensive-data case comprises $2400$ training functions, while the limited-data case uses $450$ functions as training. We employed vanilla DeepOnets with $2$ hidden layers, with $[N,N]$ neurons, in both trunk and branch. We set $N=5,10,20,40$. DeepOnets are trained with $20,000$ Adam iterations (with learning rate $0.001$ and then $0.0001$). We report the RandONet encompassing Johnson-Lindenstrauss (JL) Featured branch network (with $M=100$ neurons) and the Random Fourier Feature branch Network (RFFN) (with $M=150$ neurons for few data and with $M=2000$ neurons for many data).}
\begin{tabular}{|c|c|c|c|c|c|c|}
\hline
                      data &    ML-model       & MSE  & 5\% $L^2$ & median--$L^2$ & 95\% $L^2$ & comp. time\\
                      \hline
\multirow{2}{*}{80\%}  & DeepOnet $[5,5]$  & 2.63E$-$07  &  2.41E$-$03    & 4.50E$-$03  &   8.40E$-$03   &   1.25E$+$03  (GPU)     \\
& DeepOnet $[10,10]$  & 1.50E$-$07 &   1.71E$-$03  &  3.48E$-$03 &    5.88E$-$03    &  1.19E$+$03   (GPU)     \\
& DeepOnet $[20,20]$  & 1.83E$-$08 &   7.92E$-$04  &  1.27E$-$03  &   1.91E$-$03  &    1.36E$+$03  (GPU)     \\
& DeepOnet $[40,40]$  & 7.02E$-$09  & 5.37E$-$04  & 7.96E$-$04  & 1.13E$-$03 &     1.32E$+$03  (GPU)     \\
                      & RandONet--JL ($100$) & 1.19E$-$13 & 3.96E$-$07 & 1.89E$-$06  & 3.75E$-$06  & 3.31E$-$02  (CPU)       \\
                      & RandONet--RFFN ($2000$) & 6.52E$-$17 & 1.09E$-$08  & 4.66E$-$08 & 1.58E$-$07  & 3.91E$+$00 (CPU)       \\
                      \hline
\multirow{2}{*}{15\%} & DeepOnet $[5,5]$  & 2.30E$-$07  &  2.09E$-$03   & 4.10E$-$03  &   7.87E$-$03   &   2.97E$+$02   (GPU)     \\
& DeepOnet $[10,10]$  & 1.08E$-$07   & 1.48E$-$03  &  2.98E$-$03  &   5.02E$-$03  &    2.6170E$+$02   (GPU)     \\
& DeepOnet $[20,20]$  & 2.27E$-$08   & 8.57E$-$04  &  1.37E$-$03  &   2.24E$-$03  &    2.77E$+$02   (GPU)     \\
& DeepOnet $[40,40]$  & 2.92E$-$08  &  8.77E$-$04  &  1.45E$-$03  &   2.81E$-$03   &   2.89E$+$02   (GPU)     \\
                      & RandONet--JL ($100$) & 1.41E$-$13 & 2.82E$-$07  & 1.62E$-$06  & 5.03E$-$06 & 5.63E$-$03 (CPU) \\
                      & RandONet--RFFN ($150$) & 2.91E$-$13 & 8.55E$-$07 & 3.14E$-$06  & 9.59E$-$06  & 8.91E$-$03 (CPU) \\
                      \hline

\end{tabular}
\label{tab:pendulum}
\end{table}

The values of the parameters, of the RP-FNN based function dataset (see in Eq.~\eqref{eq:basisRPNN}), are randomly sampled from uniform distributions as follows: The values of $\bm{w},\, a_0,\, a_1,\, a_2$ are sampled uniformly in $\mathcal{U}[-0.05,0.05]$, the value of $\bm{s}$ from $ \mathcal{U}[0,500]$ and the value of $\bm{c}$ is uniformly sampled from $\mathcal{U}[0,1]$.\par
To obtain the ``ground-truth'' corresponding output functions, we employ the MATLAB solver \texttt{ode45} with absolute tolerance set to 1E$-$12 and relative tolerance set to 1E$-$10.
Here, we consider, compared to case 1, relatively smaller amplitudes in the functions, since a high forcing term can lead the system far from the initial condition. It is important to note also that in this case, the solution $v$ is not a simple primitive of the function $u$: it corresponds for a given initial condition to a nonlinear solution operator.\par
To generate the data, we used $3000$ random realizations for the values of the parameters. We consider two different sizes for the training set. As described above, we used $15$\% of the data for training and $85\%$ for testing (for the limited-data case) and $80\%$ of the data for training and $20\%$ for testing (for the extensive-data case).\par 
In Figure \ref{fig:pendulum_conv}, we report the numerical approximation accuracy for the test set in terms of the MSE and the median (and corresponding percentiles $5\%-95\%$) of the $L^2$--error. As shown, the training time of RandONets takes approximately less than one second and is performed without iterations.
In Table \ref{tab:pendulum}, we also summarize the comparative results in terms of the approximation errors, and computational times between RandONets, and vanilla DeepONets. For our illustrations, for RandOnets, we have used $M=100$ neurons for the JL embeddings and $M=150$ neurons for the RFFN, in the limited-data case, and $M=2000$ neurons in the extensive-data case.
Compared to case 1, the nonlinear RFFN embeddings perform slightly better than the linear JL embeddings, especially for higher sizes of the branch network. This is due to nonlinearity of the benchmark, yet they still both schemes exhibit comparable performance for smaller sizes of the networks. Actually, JL embeddings converge faster, generalizing better with few neurons compared to the nonlinear RFFN embeddings. 
This indicates that, despite the JL embedding being linear, the nonlinearity of the operator can be effectively approximated by the hyperbolic tangent RP-FNN based trunk embedding for the spatial locations $y$.
Also, the DeepOnets performed better than in case 1.\par
Finally, for this case study, the execution times when using RandONets (with JL and nonlinear RFFN embeddings) are $3$ to $5$ orders faster, than DeepOnets, while achieving a $3$ to $4$ orders higher numerical approximation accuracy in terms of the $L^2$ error.

\subsection{Approximation of Evolution Operators (RHS) of time-dependent PDEs}
\label{sec:benchmarkPDE}
Here, we consider some benchmark problems relative to the identification of the evolution operator, i.e., the right-hand-side (RHS) of time-dependent PDEs:
\begin{equation}
    v(x)=\frac{\partial u(x,t)}{\partial t}= \mathcal{L}(u)(x,t).
\end{equation}
The output function, the time derivative, 
(i.e., the right-hand-side of the evolutionary PDE) depends on the current state profile $u(x,t)$ at a certain time $t$. In the following examples, we do not consider the limited-data case.

\subsubsection{Case study 3: 1D Linear Diffusion-Advection-Reaction PDE}
As a first example for the learning of the evolution operator, we consider a simple 1D linear Diffusion-advection-reaction problem, described by:
\begin{equation}
    \frac{\partial u}{\partial t}=\nu \frac{\partial^2 u}{\partial x^2}+\gamma \frac{\partial u}{\partial x}+\zeta u, \qquad x\in [-1,1]
    \label{eq:LinearPDE}
\end{equation}
where $\nu=0.1$, $\gamma=0.4$ and $\zeta=-1$.

\begin{figure}[ht!]
    \centering
    \subfigure[DeepOnet ]{ 
        \includegraphics[width=0.31\textwidth]{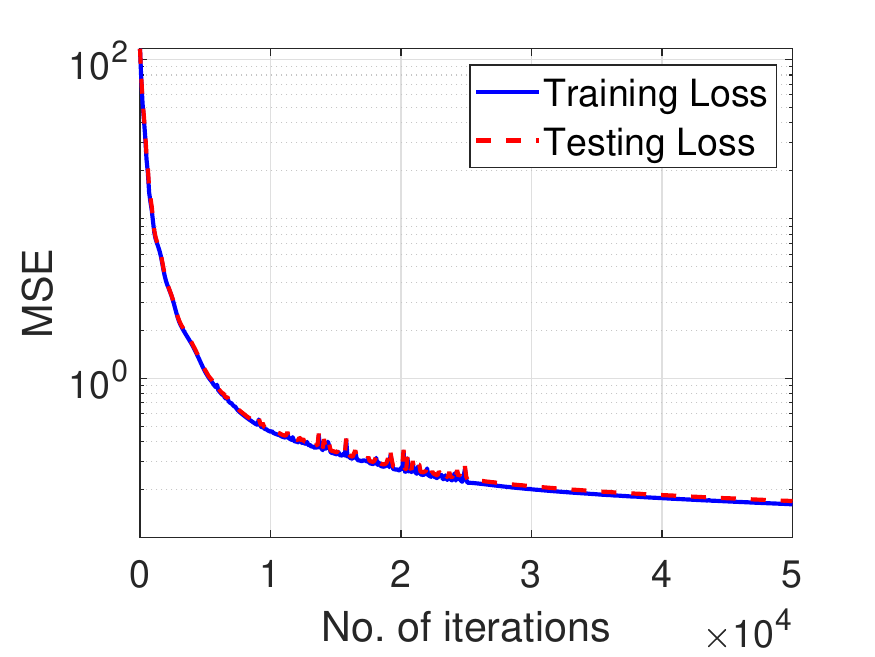}
    }
    \subfigure[RandONets ]{ 
    \includegraphics[width=0.31\textwidth]{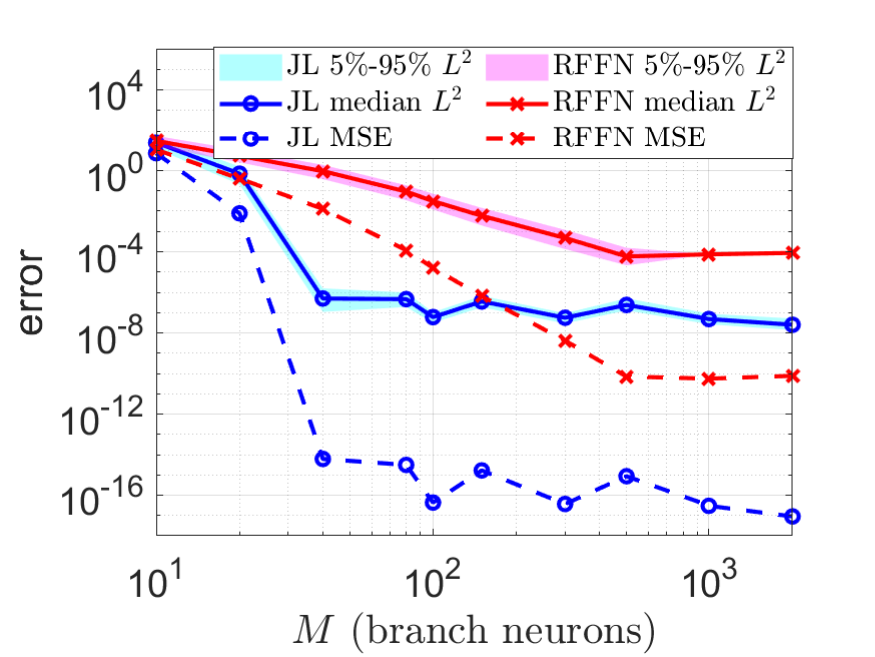}
    }
    \subfigure[RandONets ]{ 
    \includegraphics[width=0.31\textwidth]{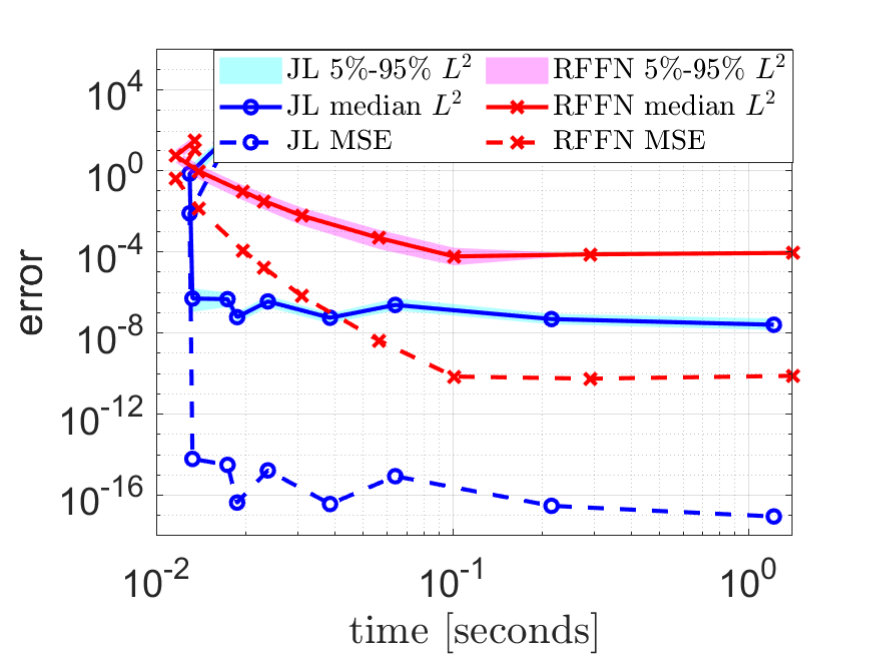}
    }
    \caption{Case study 3: 1D Diffusion-advection-reaction linear PDE in Eq.~\eqref{eq:LinearPDE}. We use $1600$ training input functions; (a) Convergence of training and test MSE of the DeepOnet with two hidden layers (indicatively) with $40$ neurons each, for both branch and trunk networks. (b), (c), MSE and $L^2$ error percentiles (median, $5\%-95\%$), of the RandONets, for different size $M$ of the branch embedding. The errors are computed w.r.t. only the output functions in the test dataset. Comparison of Johnson-Lindenstrauss (JL) random features, in Eq. \eqref{eq:JL_embedding}, with random Fourier features (RFFN) embeddings, in Eq. \eqref{eq:RFFN_embedding}. We set the size of the trunk network to $N=200$ and the grid of input points to $m=100$. Numerical approximation accuracy vs. (b) number of neurons $M$ in the hidden layer of the branch network; and (c) computational time in seconds.}
    \label{fig:DiffReac_conv}
\end{figure}
The output function can be computed analytically/symbolically based on Eq. \eqref{eq:basisRPNN}. The values of the parameters of the RP-FNN based function dataset, in Eq. \eqref{eq:basisRPNN}, are uniformly sampled as follows: $\bm{w},\, a_0,\, a_1,\, a_2 \sim \mathcal{U}[-1,1]$, $\bm{s} \sim \mathcal{U}[0,50]$ and $\bm{c} \sim \mathcal{U}[0,1]$. Here we select $w$ in a smaller range, as higher values may correspond to high derivatives resembling singularity in the second derivative.
\begin{table}[ht!]
\centering
\caption{Case study 3: 1D Diffusion-advection-reaction linear PDE (Eq.~\eqref{eq:LinearPDE}). We report the mean squared error (MSE) and percentiles (median, $5\%-95\%$) of $L^2$ error for the test set. We used $1600$ training functions. We employed a vanilla DeepOnet with $2$ hidden layers with $[N,N]$ neurons. We set $N=5,10,20,40$. DeepOnets are trained with $50,000$ Adam iterations (with learning rate $0.001$ and then $0.0001$). We report the RandONets encompassing Johnson-Lindenstrauss (JL) Featured branch network (with $M=100$ neurons) and the random Fourier features (RFFN) (with $M=500$ neurons).}
\begin{tabular}{|c|c|c|c|c|c|}
\hline
ML-model       & MSE  & 5\% $L^2$ & median--$L^2$ & 95\% $L^2$ & comp. time\\
\hline
DeepOnet $[5,5]$ & 4.74E$+$01  &  3.42E$+$01  &  6.30E$+$01  &   1.04E$+$02   &   2.18E$+$03 (GPU)\\
DeepOnet $[10,10]$ & 2.03E$+$01  &  1.88E$+$01 &   4.15E$+$01  &   6.81E$+$01   &   2.20E$+$03 (GPU)\\
DeepOnet $[20,20]$ & 1.57E$+$00 &   6.68E$+$00  &  1.09E$+$01  &   1.95E$+$01   &   2.31E$+$03 (GPU\\
DeepOnet $[40,40]$  & 1.69E$-$01 &   2.29E$+$00  &  3.73E$+$00   &  6.30E$+$00 &     2.30E$+$03 (GPU)    \\ 
RandONet--JL ($100$) & 4.33E$-$17 & 3.14E$-$08 & 5.98E$-$08 & 1.02E$-$07 & 1.83E$-$02  (CPU)       \\
RandONet--RFFN ($500$) & 7.03E$-$11 & 2.14E$-$05 & 5.87E$-$05 & 1.56E$-$04 & 1.02E$-$01 (CPU)       \\
\hline
\end{tabular}
\label{tab:DiffReac}
\end{table}

To generate the data we use $2000$ random realizations of the values of the parameters of the RP-FNNs based function dataset, as in \eqref{eq:basisRPNN}. We set $80\%$ for training and $20\%$ for testing.

In Figure \ref{fig:DiffReac_conv}, we depict the approximation accuracy w.r.t. the test set in terms of the MSE and the percentiles (median, $5\%-95\%$) of $L^2$--approximation errors. As shown, the training of all RandONets takes approximately less or around $0.1$ seconds and it is performed without iterations.
In Table \ref{tab:DiffReac}, we report the comparative results w.r.t. the numerical approximation accuracy and computational times of the RandONets, here for JL with $100$ neurons and RFFN with $500$ neurons in the branch embedding.In this case, the numerical results suggest that DeepOnets architectures exhibit limitations in achieving satisfactory accuracy, even if the operator is a simple linear RHS of a PDE. This is evident despite extensive training with Adam for $50,000$ iterations. Furthermore, we observed a concerning trend of slow convergence in DeepOnet performance as the hidden layer size ($N$) increased from $5$ to $40$ neurons. In contrast, RandOnet architectures demonstrate significantly faster learning and superior performance. RandOnet-JL achieves exceptional accuracy with a modest number of neurons ($M=100$). This is reflected in the Mean Squared Error (MSE) on the order of $10^{-17}$ and the $L^2$ error on the order of $10^{-8}$. While RandOnet-RFFN requires slightly more neurons ($M=500$) to reach its best performance. This still delivers respectable results with an MSE on the order of $10^{-11}$.\par

As a matter of fact, also for this case study, RandONets, utilizing both JL and RFFN, exhibits execution times (training times) that are on the order of $1000$ times faster, while achieving $L^2$ accuracy that is $5$ to $8$ orders of magnitude higher than the vanilla DeepONets.

\subsubsection{Case study 4: 1D viscous Burgers PDE}
We consider the nonlinear evolution operator of the Burgers' equation given by:
\begin{equation}
    v=\frac{\partial u}{\partial t}=\nu\frac{\partial^2 u}{\partial x^2}-u\frac{\partial u}{\partial x},
    \label{eq:burgers}
\end{equation}
where $\nu=0.01$.

\begin{figure}[ht!]
    \centering
    \subfigure[DeepOnet ]{ 
        \includegraphics[width=0.31\textwidth]{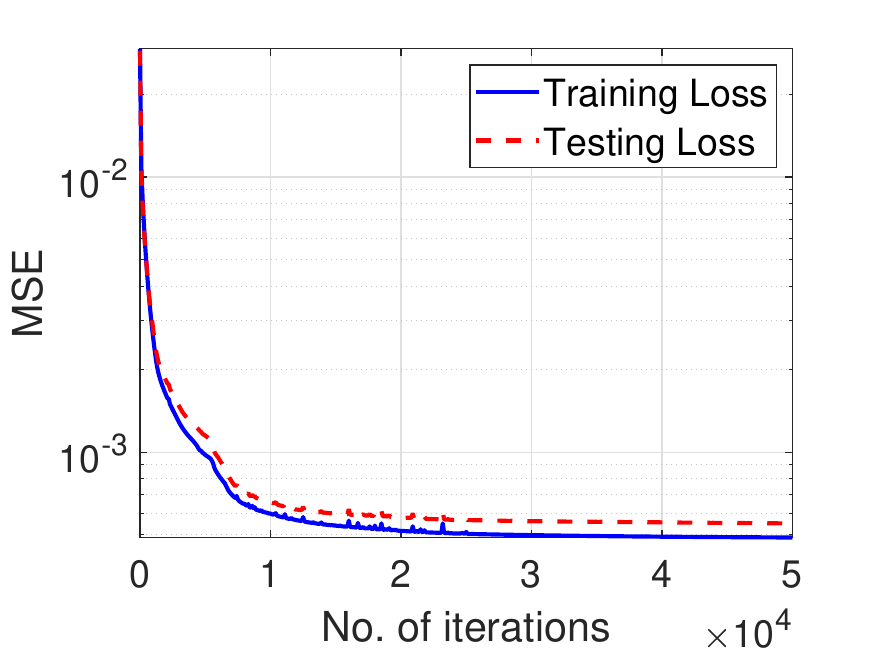}
    }
    \subfigure[RandONets ]{ 
    \includegraphics[width=0.31\textwidth]{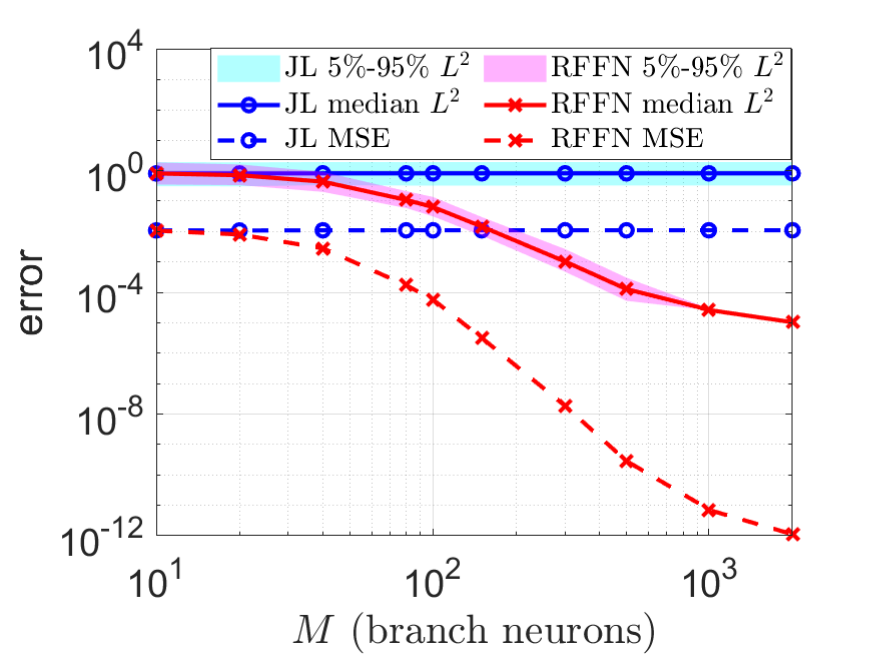}
    }
    \subfigure[RandONets ]{ 
    \includegraphics[width=0.31\textwidth]{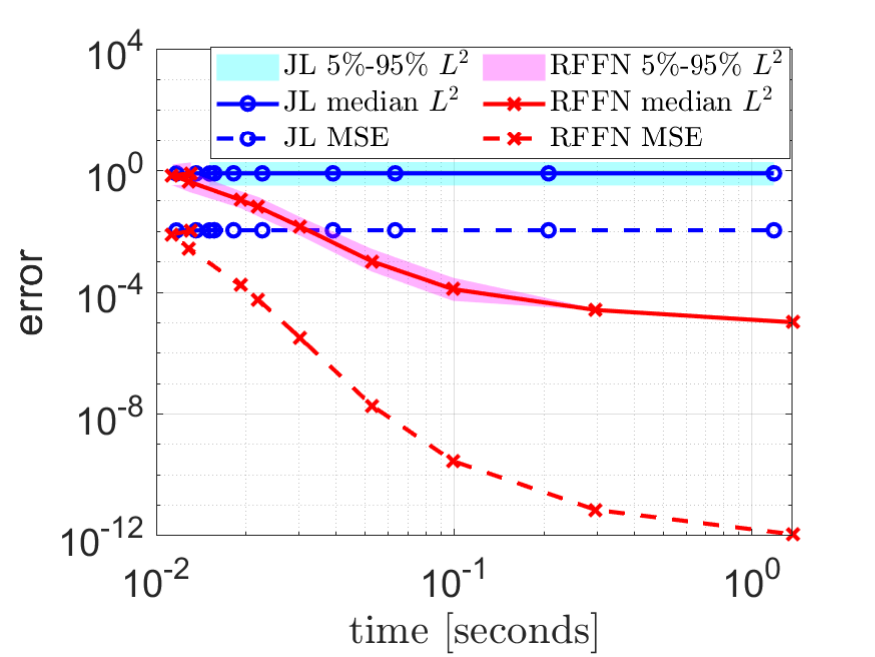}
    }
    \caption{Case study 4: 1D nonlinear Burgers' PDE (Eq.~\eqref{eq:burgers}). We used $1600$ training input functions: (a) MSE when using a vanilla DeepOnet with $2$ hidden layers with (indicatively) $40$ neurons each, for both branch and trunk networks. (b), (c), MSE and $L^2$ error percentiles (median, $5\%-95\%$), of the RandONets for different size $M$ of the branch embedding. Comparison of Johnson-Lindenstrauss random embeddings, as in Eq. \eqref{eq:JL_embedding}, with random Fourier features (RFFN) embeddings, as in Eq. \eqref{eq:RFFN_embedding}. We have set the size of the trunk network to $N=200$ and the grid of input points to $m=100$. Numerical approximation accuracy vs. (b) number of neurons $M$ in the hidden layer of the branch network; and (c) vs. computational time in seconds.}
    \label{fig:burgers_conv}
\end{figure}
The output function can be computed analytically/symbolically based on \eqref{eq:basisRPNN}. The parameters $\bm{w},\, a_0,\, a_1,\, a_2 \sim \mathcal{U}[-0.05,0.05]$, $\bm{s} \sim \mathcal{U}[0,50]$ and $\bm{c} \sim \mathcal{U}[-1,1]$, of the RP-FNN based function dataset, as in Eq.~\eqref{eq:basisRPNN}, to represent the functional space are (element-wise) uniformly distributed. Here we select $w$ in a smaller range, as higher values may correspond to high second derivatives approaching singularity.
\begin{table}[ht!]
\centering
\caption{Case study 4:  Burgers' Nonlinear PDE in Eq.~\eqref{eq:burgers}. We report Mean Squared Error (MSE) and percentiles (median, $5\%-95\%)$ of $L^2$ approximation errors, for the test set. We use $1600$ training functions. We employ a DeepOnet with $2$ hidden layers with $[N,N]$ neurons in both the branch and trunk. We set $N=5,10,20,40$. DeepOnets are trained with $50,000$ Adam iterations (with learning rate $0.001$ and then $0.0001$). We report the RandONets encompassing Johnson-Lindenstrauss (JL) Featured branch network (with $M=40$ neurons) and the Random Fourier Feature branch Network (RFFN) (with $M=2000$ neurons).}
\begin{tabular}{|c|c|c|c|c|c|}
\hline
ML-model  & MSE  & 5\% $L^2$ & median--$L^2$ & 95\% $L^2$ & comp. time\\
\hline
DeepOnet $[5,5]$  & 9.00E$-$03  &  3.70E$-$01   & 7.60E$-$01   &  1.66E$+$00   &   2.16E$+$03 (GPU)    \\
DeepOnet $[10,10]$  & 4.75E$-$03  &  3.43E$-$01  &  5.59E$-$01   &  1.20E$+$00   &   2.01E$+$03 (GPU)    \\
DeepOnet $[20,20]$  & 1.51E$-$03 &   2.24E$-$01  &  3.28E$-$01  &   6.16E$-$01  &    2.40E$+$03  (GPU)    \\
DeepOnet $[40,40]$  & 5.50E$-$04  &  1.30E$-$01  &  2.03E$-$01  &   3.82E$-$01   &   2.34E$+$03 (GPU)    \\
RandONet--JL ($40$) & 1.09E$-$02 & 3.32E$-$01 & 8.11E$-$01 & 1.91E$+$00 & 1.29E$-$02  (CPU)       \\
RandONet--RFFN ($2000$) & 1.12E$-$12 & 1.01E$-$05 & 1.04E$-$05 & 1.19E$-$05 & 1.51E$+$00 (CPU)       \\
\hline
\end{tabular}
\label{tab:burgers}
\end{table}

To generate the data, we used $2000$ random realizations of the parameters. We set $80\%$ for training and $20\%$ for testing.
In Figure \ref{fig:DiffReac_conv}, we report the accuracy w.r.t. the test set in terms of the MSE and the median (and percentiles $5\%-95\%$) of $L^2$--errors. As shown, the training of all RandONets takes approximately less or around one second.
In Table \ref{tab:DiffReac}, we also report the comparison results in terms of the numerical approximation accuracy and computational times.\par
Due to the inherent non-linearity of this example, linear JL random embeddings exhibit limitations in efficiently approximating the non-linear operator. This observation aligns with the theoretical understanding of JL embeddings being most effective in capturing linear relationships. Unlike case study 2, the non-linearity within only the trunk architecture appears insufficient for this specific problem. Therefore, incorporating non-linearity also in the branch embedding becomes crucial for achieving optimal performance.

Interestingly, the performance of JL embeddings approaches that of fully trained, entirely non-linear vanilla DeepOnets. While DeepOnets can achieve a minimum Mean Squared Error (MSE) on the order of 1E$-$04 and an $L^2$ error on the order of 1E$-$01, their performance is not significantly better than the JL approach. In contrast, the RandOnet-RFFN architecture emerges as the clear leader in this specific case study. It achieves a remarkably low MSE on the order of 1E$-$12, demonstrating its superior capability in handling the non-linearities present in this example.\par
Also for this case study, RandONets, utilizing both JL and RP-FFN random embeddings, demonstrates execution times (training times) that are $3$ to $5$ order times faster, while achieving $L^2$ accuracy that is $4$ orders of magnitude higher in the case of RFFN, and of a similar level of accuracy in the case of JL random embeddings.

\subsubsection{Case study 5: 1D Allen-Cahn phase-field PDE}
Here we consider the nonlinear evolution operator of the Allen-Cahn equation, described by:
\begin{equation}
    v=\frac{\partial u}{\partial t}=\nu\frac{\partial^2 u}{\partial x^2}+(u-u^3)
    \label{eq:AllenCahn}
\end{equation}
where we set the parameter $\nu=0.01$.

\begin{figure}[ht!]
    \centering
    \subfigure[DeepOnet ]{ 
        \includegraphics[width=0.31\textwidth]{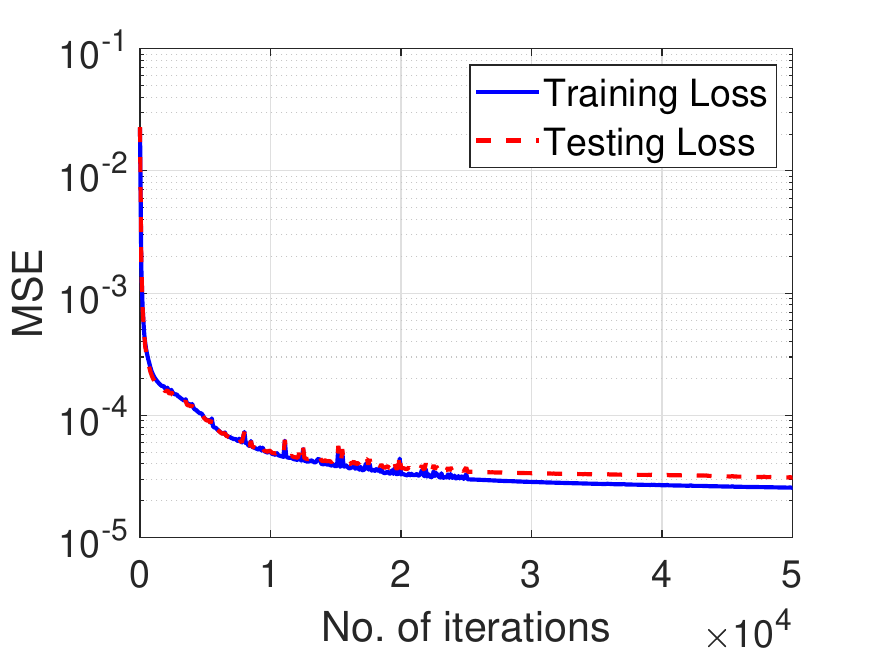}
    }
    \subfigure[RandONets ]{ 
    \includegraphics[width=0.31\textwidth]{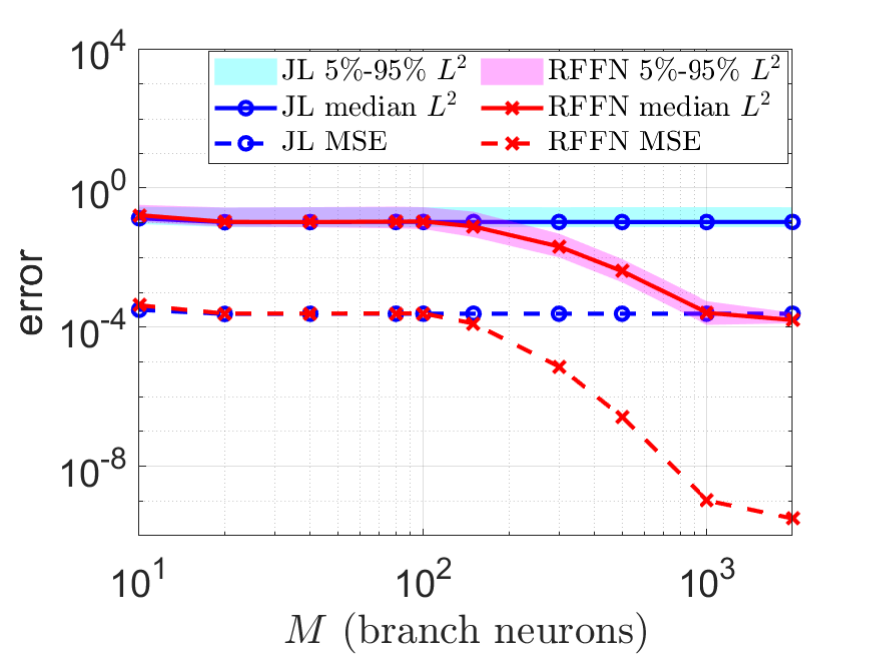}
    }
    \subfigure[RandONets ]{ 
    \includegraphics[width=0.31\textwidth]{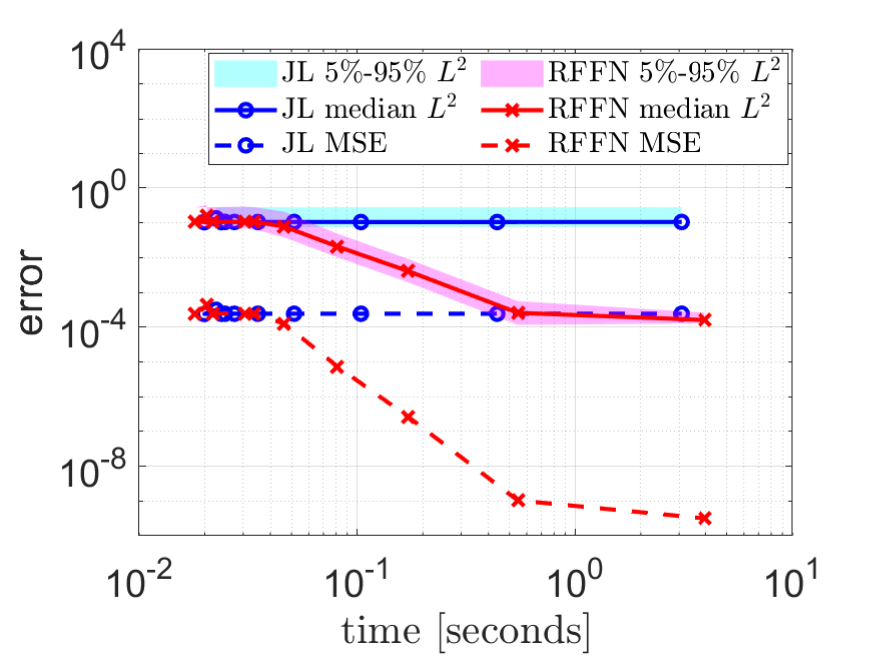}
    }
    \caption{Case study 5: 1D Allen-Cahn phase-field PDE in Eq.~\eqref{eq:AllenCahn}. We use $800$ training input functions; (a) Convergence of training and test MSE of the DeepOnet with $2$ hidden layer (indicatively) with $40$ neurons each, for both branch and trunk networks. (b), (c), MSE and $L^2$ approximation errors percentiles (median, $5\%-95\%$), of the RandONets, for different size $M$ of the branch embedding. The errors are computed w.r.t only the output functions in the test dataset. Comparison of Johnson-Lindenstrauss random features, as in Eq. \eqref{eq:JL_embedding}, with Random Fourier Feature Networks (RFFN), as in Eq. \eqref{eq:RFFN_embedding}. We set the size of the Trunk network to $N=200$ and the grid of input points to $m=100$. Numerical approximation accuracy vs. (b) number of neurons $M$ in the hidden layer of the branch network; and (c) vs. computational time in seconds.}
    \label{fig:AllenCahn_conv}
\end{figure}

The output function can be computed analytically/symbolically based on \eqref{eq:basisRPNN}. The parameters $\bm{w},\, a_0,\, a_1,\, a_2 \sim \mathcal{U}[-0.05,0.05]$, $\bm{s} \sim \mathcal{U}[0,50]$ and $\bm{c} \sim \mathcal{U}[-1,1]$, of the RP-FNN based function dataset, as in Eq.~\eqref{eq:basisRPNN}, to represent the functional space are (element-wise) uniformly distributed. Here, similarly to case study 4, we select $w$ in a smaller range as higher values may correspond to high second derivatives approaching singularity.


\begin{table}[ht!]
\centering
\caption{Case study 5: 1D Allen-Cahn phase-field PDE in Eq.~\eqref{eq:AllenCahn}. We report the mean Squared Error (MSE) and percentiles (median, $5\%-95\%$) of the $L^2$ approximation error for the test set. We use $2400$ training functions. Here, we depict, indicatively, the results with a vanilla DeepOnet with $2$ hidden layers with $[N,N]$ neurons. We set $N=5,10,20,40$. DeepOnets are trained with $50'000$ Adam iterations (with learning rate $0.001$ and then $0.0001$). We report the RandONets encompassing Johnson-Lindenstrauss (JL) Featured branch network (with $M=40$ neurons) and the Random Fourier Feature branch Network (RFFN) (with $M=2000$).}
\begin{tabular}{|c|c|c|c|c|c|}
\hline
ML-model       & MSE  & 5\% $L^2$ & median--$L^2$ & 95\% $L^2$ & comp. time\\
\hline
DeepOnet $[5,5]$  & 1.75E$-$03  &  1.57E$-$01 &   3.50E$-$01  &   7.28E$-$01  &    3.20E$+$03 (GPU)    \\
DeepOnet $[10,10]$  & 1.60E$-$04  &  6.46E$-$02  &  1.04E$-$01  &   2.12E$-$01   &   3.03E$+$03 (GPU)\\
DeepOnet $[20,20]$  & 5.62E$-$05  &  3.49E$-$02  &  5.52E$-$02  &   1.28E$-$01   &   3.39E$+$03 (GPU)\\
DeepOnet $[40,40]$  & 3.10E$-$05  &  2.39E$-$02   & 4.06E$-$02  &   9.38E$-$02   &   3.49E$+$03 (GPU)    \\
RandONet--JL ($40$) & 2.42E$-$04 & 7.61E$-$02 & 1.04E$-$01 & 2.77E$-$01 & 2.32E$-$02  (CPU)       \\
RandONet--RFFN ($2000$) & 3.15E$-$10 & 1.28E$-$04 & 1.60E$-$04 & 2.60E$-$04 & 3.20E$+$00 (CPU)       \\
\hline
\end{tabular}
\label{tab:AllenCahn}
\end{table}

To generate the data we use $3000$ random realizations of the parameters. We set $80\%$ for training and $20\%$ for testing.
We report the accuracy w.r.t. the test set in terms of MSE and the percentiles (median,$5\%-95\%$) of $L^2$ approximation errors in Figure \ref{fig:DiffReac_conv}. As shown, the training of the all RandONets takes approximately less or around 3 seconds and it is performed without iterations.
In Table \ref{tab:DiffReac}, we also report the comparison results with the vanilla DeepONets, denoted as $[N,N]$. We set $N=5,10,20,40$. \par
Consistent with our observations in case study 4, linear JL random embeddings exhibit limitations in efficiently approximating the non-linear Allen-Cahn operator. Like in case Study 4, the performance of JL embeddings approaches that of fully trained, entirely non-linear vanilla DeepOnets. While DeepOnets can achieve a minimum Mean Squared Error (MSE) on the order of 1E$-$04 and an $L^2$ error on the order of 1E$-$01, their performance is not significantly better than the RandOnets-JL approach.
Once again, the RandOnet-RFFN architecture emerges as the superior method. It achieves a remarkably low MSE on the order of 1E$-$10, demonstrating its capability in handling the non-linearities present in this example.

Finally, RandONets, utilizing both JL and RFFN embeddings, exhibit execution times (training times) that are of $3$ to $5$ orders faster, while achieving $L^2$ approximation accuracy that is of $2$ orders of magnitude higher when using RFFN embeddings, and of the same level of accuracy when using the linear JL random embeddings.

\section{Conclusions}
\label{sec:conclusion}
In this work, we presented RandONets, a framework based on DeepONets \cite{lu2021learning} and the celebrated by now paper due to DeepOnet, of Chen \& Chen \cite{chen1995universal} to approximate linear and nonlinear operators. Our work builds on three keystones: (a) random embeddings by Johnson \& Lindenstrauss (JL) for linear projections, and of Rahimi \& Recht \cite{rahimi2008uniform,rahimi2007random} for nonlinear random embeddings, (b) random projection neural networks whose ``birth'' can be traced back in early '90s \cite{schmidt1992feed,barron1993universal,pao1994learning}, and, (c) niche numerical analysis for the solution of the linear least-squares problem.  First, based on the above, we prove the universal approximation property of RandONets. We furthermore assess their performance by comparing them with vanilla DeepONets on various benchmark problems, including, simple problems of the approximation of the solution operator of ODEs, and linear and nonlinear evolution operators (right-hand-sides) of PDEs. We show that for these benchmark problems, and for the particular task, the proposed scheme outperforms the vanilla DeepONets in both computational cost and numerical accuracy. In particular, we show that RandONets with JL random embeddings are unbeatable when approximating linear evolution operators of PDEs, resulting to almost machine-precision accuracy for aligned data. For the benchmark problems considered for nonlinear evolution operators of PDEs, such as the 1D nonlinear viscous Burgers PDE, and the 1D phase-field Allen-Cahn PDE, RandONets with nonlinear random embeddings are on the order of $10^2-10^3$ more accurate and $10^3$ times faster than the vanilla DeepONets. Here, as in the celebrated paper introducing DeepONet\cite{lu2021learning}, we aimed at introducing to the community RandONets. So an extensive comparison, with more advanced versions of DeepONets and other approaches such as FNOs and others, is beyond the scope of this current work. However, in a following work, we aim at assessing the performance of RandONets considering high-dimensional nonlinear operators, and the approximation of the solution operator of PDEs, thus performing  ``a comprehensive and fair comparison'' of the various machine learning schemes as performed also in \cite{lu2022comprehensive}.\par 
We believe that our work will trigger further advances in the field, paving the way for further exploration of how niche numerical analysis can enhance the capabilities of powerful machine learning methodologies such as DeepOnets.

\section*{Acknowledgements}
I.G.K. acknowledges partial support from the US AFOSR FA9550-21-0317 and the US Department of Energy SA22-0052-S001. C.S. acknowledges partial support from the PNRR MUR, projects PE0000013-Future Artificial Intelligence Research-FAIR \& CN0000013 CN HPC - National Centre for HPC, Big Data and Quantum Computing. A.N.Y. acknowledges the use of resources from the Stochastic Modelling and 
Applications Laboratory, AUEB.  

\section*{Competing Interests}
The authors declare that there is no known competing financial interests or personal relationships that could have
appeared to influence the work reported in this paper.


\section*{Authorship contribution statement - CRediT}
\textbf{G.F.}: Conceptualization, Data curation, Mathematical Analysis, Investigation, Methodology, Software, Validation, Visualization, Writing – original draft, Writing – review \& editing.\\
\textbf{I.G.K.}: Methodology, Supervision, Validation, Writing – review \& editing.\\
\textbf{C.S.}: Conceptualization, Mathematical analysis,  Methodology, Supervision, Validation, Writing – original draft, Writing – review \& editing\\
\textbf{A.N.Y.}: Mathematical analysis, Methodology, Supervision, Validation, Writing – review \& editing.


\end{document}